\documentclass{article}

\usepackage[nonatbib, preprint]{neurips_2021}

\usepackage[utf8]{inputenc} 
\usepackage[T1]{fontenc}    

\usepackage{hyperref}       

\hypersetup{
    colorlinks=true,
    citecolor=yale,
    linkcolor=yale,
    filecolor=magenta,      
    urlcolor=Black,
}

\usepackage{url}            
\usepackage{booktabs}       
\usepackage{amsfonts}       
\usepackage{nicefrac}       
\usepackage{microtype}      
\usepackage[svgnames]{xcolor}
\definecolor{yale}{RGB}{14,77,146}

\usepackage{mathtools}
\usepackage{amsmath}
\usepackage{amssymb}
\usepackage{amsthm}
\usepackage{bbm}
\usepackage{dsfont}
\usepackage{bm}
\usepackage{graphicx}
\usepackage{subfigure}
\usepackage{theoremref}

\usepackage{tikz}  
\usetikzlibrary{automata, positioning}
\usepackage[ruled,linesnumbered]{algorithm2e}

\title{Navigating to the Best Policy\\ in Markov Decision Processes}

\author{%
  Aymen Al Marjani
  \\ UMPA, ENS Lyon\\ 
  Lyon, France \\ \texttt{aymen.al\_marjani@ens-lyon.fr} \\
   \AND
   Aurélien Garivier \\
   UMPA, CNRS, INRIA, ENS Lyon \\
   Lyon, France \\
   \texttt{aurelien.garivier@ens-lyon.fr} \\
   \And
  Alexandre Proutiere \\
   EECS and Digital Futures \\
   KTH Royal Institute of Technology, Sweden \\
   \texttt{alepro@kth.se} \\
}

\newtheorem{theorem}{Theorem}
\newtheorem{corollary}{Corollary}
\newtheorem{lemma}[theorem]{Lemma}
\newtheorem{remark}{Remark}

\newtheorem{proposition}[theorem]{Proposition}

\newtheorem{assumption}{Assumption}

\newcommand{\Alt}[1]{\operatorname{Alt}(#1)}
\newcommand{\spn}[1]{\operatorname{sp}(#1)}
\newcommand{\KL}[2]{\operatorname{KL}(#1,\,#2)}

\newcommand{\starQ}[1]{Q_{#1}^\star}
\newcommand{\starV}[1]{V_{#1}^\star}
\newcommand{\rvline}{\hspace*{-\arraycolsep}\vline\hspace*{-\arraycolsep}}

\DeclareMathOperator*{\argmax}{arg\,max}
\DeclareMathOperator*{\argmin}{arg\,min}

\DeclarePairedDelimiter\floor{\lfloor}{\rfloor}

\newcommand{\E}{\mathbb{E}}
\renewcommand{\P}{\mathbb{P}}

\newcommand{\acal}{\mathcal{A}}
\newcommand{\bcal}{\mathcal{B}}
\newcommand{\ccal}{\mathcal{C}}
\newcommand{\ecal}{\mathcal{E}}
\newcommand{\fcal}{\mathcal{F}}
\newcommand{\mcal}{\mathcal{M}}
\newcommand{\lcal}{\mathcal{L}}
\newcommand{\ocal}{\mathcal{O}}
\newcommand{\scal}{\mathcal{S}}
\newcommand{\zcal}{\mathcal{Z}}

\newcommand{\indicator}{\mathbbm{1}}

\newcommand{\Alg}{\mathbb{A}}

\newcommand{\bomega}{{\boldsymbol \omega}}
\newcommand{\bN}{{\boldsymbol N}}

\newcommand{\transpose}{^\mathsf{\scriptscriptstyle T}}
\newcommand{\norm}[1]{\left\lVert#1\right\rVert}
\newcommand{\kl}{\operatorname{kl}}

\renewcommand{\epsilon}{\varepsilon}

\begin{document}

\maketitle

\begin{abstract}
	We investigate the classical active pure exploration problem in Markov Decision Processes, where the agent sequentially selects actions and, from the resulting system trajectory, aims at identifying the best policy as fast as possible. We propose a problem-dependent lower bound on the average number of steps required before a correct answer can be given with probability at least $1-\delta$. We further provide the first algorithm with an instance-specific sample complexity in this setting. This algorithm addresses the general case of communicating MDPs; we also propose a variant with a reduced exploration rate (and hence faster convergence) under an additional ergodicity assumption. 
	This work extends previous results relative to the \emph{generative setting}~\cite{pmlr-v139-marjani21a}, where the agent could at each step query the random outcome of any (state, action) pair. In contrast, we show here how to deal with the \emph{navigation constraints}, induced by the \emph{online setting}. Our analysis relies on an ergodic theorem for non-homogeneous Markov chains which we consider of wide interest in the analysis of Markov Decision Processes.
\end{abstract}

\section{Introduction}\label{sec:intro}

Somewhat surprisingly, learning in a Markov Decision Process is most often considered under the performance criteria of \emph{consistency} or \emph{regret minimization} (see e.g.~\cite{Sutton1998,csaba10,torCsaba2020} and references therein). Regret minimization (see e.g.~\cite{auerJaksch10,klucrl10}) is particularly relevant when the rewards accumulated during the learning phase are important. This is however not always the case: for example, when learning a game (whether Go, chess, Atari, or whatever), winning or losing during the learning phase does not really matter. One may intuitively think that sometimes getting into difficulty on purpose so as to observe unheard situations can significantly accelerate the learning process. Another example is the training of robot prototypes in the factory: a reasonably good policy is first searched, regardless of the losses incurred, that can serve as an initialization for a second phase regret-minimization mode that starts when the robot is deployed. It is hence also of great practical importance to study the sample complexity of learning, and to work on strategies that might improve, in this perspective, on regret minimizing algorithms. 

In this work, we are interested in the best policy identification (BPI) problem for infinite-horizon discounted MDPs. This framework was introduced by \cite{Fiechter1994EfficientRL} under the name of \textit{PAC-RL}. In BPI the algorithm explores the MDP until it has gathered enough samples to return an $\varepsilon$-optimal policy with probability at least $1-\delta$. Crucially, the algorithm halts at a \textit{random time step}, determined by a \textit{stopping rule} which guarantees that the probability of returning a wrong answer is less that $\delta$. The optimality of BPI algorithms is measured through their \textit{sample complexity}, defined as their expected stopping time. Best policy identification in MDPs has been mostly investigated under the lens of minimax-optimality. The minimax framework may be overly pessimistic by accounting for the worst possible MDP, whereas an algorithm with instance-specific guarantees can adapt its exploration procedure to the hardness of the MDP instance that it faces. Recently, a few works made the first attempts towards understanding the instance-specific sample complexity of reinforcement learning. However, they typically either make simplifying assumptions such as access to a generative model \cite{BESPOKE2019,pmlr-v139-marjani21a}, or restrict their attention to episodic MDPs \cite{wagenmaker2021noregret}. In practice however, the samples gathered rather correspond to a single, possibly infinite, trajectory of the system that we wish to control. This motivates our study of the full online setting where observations are only obtained by navigating in the MDP, that is by sequential choices of actions and following the transitions of the MDP.

{\bf Our contributions.} \cite{pmlr-v139-marjani21a} recently proposed an information-theoretical complexity analysis for MDPs in the case of access to a generative model. Here we extend their results to the online setting. Our main goal is to understand how the online learning scheme affects the sample complexity compared to the easier case where we have a generative model. A natural first step consists in understanding how the first order term $T^\star\log(1/\delta)$ changes. Thus we only focus on the asymptotic regime $\delta\to 0$. Our key contributions can be summarized as follows: 
\begin{itemize}
\item First, we adapt the lower bound of \cite{pmlr-v139-marjani21a} to the online setting (Proposition \ref{prop:LB2}). The new bound also writes as the value of a zero-sum two-player game between nature and the algorithm, where the loss function remains the same, but where the set of possible strategies for the algorithm is restricted to a subset $\Omega(\mcal)$ of  the simplex of dimension $SA-1$. We refer to the constraints defining $\Omega(\mcal)$ as the \emph{navigation constraints}.
\item We propose MDP-NaS, the first algorithm for the online setting\footnote{Before publication of this work, but after a preprint was available online, \cite{wagenmaker2021noregret} proposed another algorithm with instance-dependent guarantees.} with instance-dependent bounds on its sample complexity in the asymptotic regime $\delta \to 0$ (Theorem \ref{theorem:expectation}). A major challenge lies in the design of a sampling rule that guarantees that the sampling frequency of state-action pairs $\big(N_{sa}(t)/t\big)_{s,a}$ converges to some target oracle allocation $\omega^\star \in \Omega(\mcal)$. Indeed, since we can no longer choose the next state of the agent, the tracking procedure which was developed by \cite{garivier16a} for multi-armed bandits and used in \cite{pmlr-v139-marjani21a} for MDPs with a generative model can no longer be applied in our setting. We propose a new sampling rule which performs exploration according to a mixture of the uniform policy and a plug-in estimate of the \textit{oracle policy} (the policy whose stationary state-action distribution is $\omega^\star$) and prove that it satisfies the requirement above. The analysis of our sampling rule relies on an ergodic theorem for non-homogeneous Markov chains of independent interest (Proposition~\ref{prop:ergodic}).
\item We investigate, depending on the communication properties of the ground-truth instance $\mcal$, what is the minimal forced-exploration rate in our sampling rule that guarantees the consistency of the plug-in estimator of the oracle policy. Our findings imply that when $\mcal$ is ergodic, an exploration rate as low as $1/\sqrt{t}$ is sufficient. However, when $\mcal$ is only assumed to be communicating, one is obliged to resort to a much more conservative exploration rate of $t^{-\frac{1}{m+1}}$, where $m$ is a parameter defined in Lemma \ref{lemma:visits} that may scales as large as $S-1$ in the worst case. 
\item Finally, our stopping rule represents the first implementation of the Generalized Likelihood Ratio test for MDPs. Notably, we circumvent the need to solve the max-min program of the lower bound exactly, and show how an upper bound of the best-response problem, such as the one derived in \cite{pmlr-v139-marjani21a}, can be used to perform a GLR test. This improves upon the sample complexity bound that one obtains using the KL-Ball stopping rule of \cite{pmlr-v139-marjani21a} by at least a factor of $2$\footnote{Note that the stopping rule is independent of the sampling rule and thus can be used in both the generative and the online settings. Furthermore, one may even obtain an improved factor of $4$ by using a deviation inequality for the full distribution of (reward, next-state) instead of a union bound of deviation inequalities for each marginal distribution.}.
\end{itemize}

\subsection{Related work}

{\bf Minimax Best-Policy Identification.} BPI in the online setting has been investigated recently by a number of works with minimax sample complexity bounds. In the case of episodic MDPs \cite{pmlr-v132-kaufmann21a}, \cite{Menard2020FastAL} proposed algorithms that identify an $\varepsilon$-optimal policy at the initial state w.h.p. In contrast, in the case of infinite-horizon MDPs one is rather interested in finding a good policy at {\it every} state. Recent works provide convergence analysis for Q-learning \cite{li2020asynchronous} or policy gradient algorithms \cite{PCPG_NEURIPS2020}, \cite{Feng2021ProvablyCO}, \cite{Zanette2021CautiouslyOP}. Their results typically state that if the algorithm is fed with the appropriate hyperparameters, it can return an $\varepsilon$-optimal policy w.h.p. after collecting a polynomial number of samples. In practice, a pure exploration agent needs a stopping rule to determine when to halt the learning process. In particular, the question of how to tune the number of iterations of these algorithms without prior knowledge of the ground truth instance remains open.\footnote{Here we chose not to mention works that investigate the sample complexity in the PAC-MDP framework \cite{kakade2003sample}. Indeed, in this framework the sample complexity is rather defined as the the number of episodes where the algorithm does not play an $\varepsilon$-optimal policy. As explained in \cite{Domingues2021a}, this objective is closer to regret minimization than to pure exploration.}

{\bf Generative Model.}
A large body of literature focuses on the so-called {\it generative model}, where in each step, the algorithm may query a sample (i.e., observe a reward and a next-state drawn from the rewards and transitions distributions respectively) from any given state-action pair \cite{kearnsSingh98},\cite{KearnsMansour99}, \cite{EMM06}, \cite{azar2013minimax},\cite{Dietterich13}, \cite{Wang17}, \cite{Sidford18a}, \cite{BESPOKE2019}, \cite{pmlr-v125-agarwal20b}, \cite{li2020breaking}, \cite{li2021qlearning}, \cite{pmlr-v139-marjani21a}.

{\bf Instance-specific bounds.} Instance-optimal algorithms for Best Arm Identification in multi armed bandits (MDPs with one state) have been obtained independently by \cite{garivier16a},\cite{Russo16}. Here, we extend their information-theoretical approach to the problem of Best Policy Identification in MDPs. More recently, \cite{wagenmaker2021noregret} provided an algorithm for BPI in episodic MDPs with instance-specific sample complexity. A more detailed comparison with \cite{wagenmaker2021noregret} can be found in Section~\ref{sec:MOCA}.

{\bf Outline.} The rest of the paper is organized as follows: After introducing the setting and giving some notation and definitions in Section~\ref{sec:prelim}, we derive in Section~\ref{sec:lb} a lower bound on the time required by any algorithm navigating the MDP until it is able to identify the best policy with probability at least $1-\delta$. The algorithm is presented in Section~\ref{sec:alg}. Section~\ref{sec:analysis} contains our main results along with a sketch of the analysis. Finally, in Section~\ref{sec:MOCA} we compare our results with MOCA, the only other algorithm (to the best of our knowledge) with problem-dependent guarantees in the online setting. Most technical results and proofs are given in the appendix. 


\section{Setting and notation}\label{sec:prelim}

{\bf Discounted MDPs.} We consider infinite horizon MDPs with discount and finite state and action spaces ${\cal S}$ and ${\cal A}$. Such an MDP ${\cal M}$ is defined through its transition and reward distributions $p_{\cal M}(\cdot |s,a)$ and $q_{\cal M}(\cdot|s,a)$ for all $(s,a)$. For simplicity, $q_{\cal M}(\cdot|s,a)$ will denote the density of the reward distribution w.r.t. a positive measure $\lambda$ with support included in $[0,1]$. Specifically, $p_{\cal M}(s' |s,a)$ denotes the probability of transitioning to state $s'$ after playing action $a$ at state $s$ while $R(s,a)$ is the random instantaneous reward that is collected. Finally, $\gamma \in [0,1)$ is a discounting factor. We look for an optimal control policy $\pi^\star: \scal \to \acal$ maximizing the long-term discounted reward: $V_{\cal M}^\pi(s) = \mathbb{E}_{{\cal M},\pi}[\sum_{t=0}^\infty \gamma^t R(s_t, \pi(s_t))]$, where $\mathbb{E}_{{\cal M},\pi}$ denotes the expectation w.r.t the randomness in the rewards and the trajectory when the policy $\pi$ is used. Classically, we denote by $V_{\cal M}^\star$ the optimal value function of ${\cal M}$ and by $Q^\star_{\mcal}$ the optimal $Q$-value function of ${\cal M}$. $\Pi^\star(\mcal) = \{ \pi: \scal \to \acal,\ V_{\cal M}^\pi = V_{\cal M}^\star\}$ denotes the set of optimal policies for $\mcal$.

{\bf Problem-dependent quantities.} The sub-optimality gap of action $a$ in state $s$ is defined as $\Delta_{sa} = \starV{\mcal}(s) - \starQ{\mcal}(s,a)$. Let $\Delta_{\min}=\min_{s,a \neq \pi^\star(s)}\Delta_{sa}$ be the minimum gap and $\spn{\starV{\mcal}}=\max_{s,s'}|\starV{\mcal}(s)-\starV{\mcal}(s')|$ be the span of 
$\starV{\mcal}$. Finally, we introduce the variance of the value function in $(s,a)$ as $\mathrm{Var}_{p(s,a)}[\starV{\mcal}] = \mathrm{Var}_{s'\sim p_{\mcal}(.|s,a)}[\starV{\mcal}(s')]$. 

{\bf Active pure exploration with fixed confidence.} When ${\cal M}$ is unknown, we wish to devise a learning algorithm $\Alg$ identifying, from a single trajectory, an optimal policy as quickly as possible with some given level of confidence. Formally, such an algorithm consists of (i) a sampling rule, selecting in each round $t$ in an adaptive manner the action $a_t$ to be played; $a_t$ depends on past observations, it is ${\cal F}_t^\Alg$-measurable where  ${\cal F}_t^\Alg=\sigma(s_0,a_0,R_0\ldots,s_{t-1},a_{t-1},R_{t-1},s_t)$ is the $\sigma$-algebra generated by observations up to time $t$; (ii) a stopping rule $\tau_\delta$, a stopping time w.r.t. the filtration $({\cal F}_t^\Alg)_{t\ge 0}$, deciding when to stop collecting observations; (iii) a decision rule returning $\hat{\pi}^\star_{\tau_\delta}$ an estimated optimal policy. 

An algorithm $\Alg$ is $\delta$-Probably Correct ($\delta$-PC) over some set $\mathbb{M}$ of MDPs, if for any ${\cal M}\in \mathbb{M}$, it returns (in finite time) an optimal policy with probability at least $1-\delta$. In this paper, we aim to devise a $\delta$-PC algorithm $\Alg$ with minimal sample complexity $\mathbb{E}_{{\cal M},\Alg}[\tau_\delta]$. We make the following assumption.

\begin{assumption} We consider the set $\mathbb{M}$ of {\it communicating} MDPs with a unique optimal policy.
\label{assumption:uniqueness}
\end{assumption}
We justify the above assumption as follows. (i) We restrict our attention to the case where ${\cal M}$ is communicating, for otherwise, if it is Multichain there would be a non-zero probability that the algorithm enters a subclass of states from which there is no possible comeback. In this case it becomes impossible to identify the \textit{global} optimal policy\footnote{Unless we modify the objective to finding the optimal policy in this subchain.}. (ii) About the uniqueness of the optimal policy, treating the case of MDPs with multiple optimal policies, or that of $\epsilon$-optimal policy identification, requires the use of more involved Overlapping Hypothesis tests, which is already challenging in multi-armed bandits (MDPs with a single state) \cite{Garivier2019NonAsymptoticST}. We will analyze how to remove this assumption in future work.

{\bf Notation.} $\zcal\triangleq\scal\times\acal$ is the state-action space. $\Sigma \triangleq \{\omega \in \mathbb{R}_{+}^{S\times A}:\ \sum_{s,a} \omega_{s a} = 1\}$, the simplex of dimension $SA -1$. $N_{sa}(t)$ denotes the number of times the state-action pair $(s,a)$ has been visited up to the end of round $t$. We also introduce $N_s(t)=\sum_{a \in \acal} N_{sa}(t)$. Similarly, for a vector $\omega \in \Sigma$ we will denote  $\omega_s \triangleq \sum_{a \in \acal} \omega_{s a}$. For a matrix $A \in \mathbb{R}^{n\times m}$, the infinity norm is defined as $\norm{A}_\infty \triangleq \max_{1\leq i\leq n} \sum_{j=1}^m |a_{i,j}|$. The KL divergence between two probability distributions $P$ and $Q$ on some discrete space $\scal$ is defined as: $KL(P\| Q) \triangleq \sum_{s \in \scal} P(s)\log(\frac{P(s)}{Q(s)})$. For Bernoulli distributions of respective means $p$ and $q$, the KL divergence is denoted by $\kl (p,q)$. For distributions over $[0,1]$ defined through their densities $p$ and $q$ w.r.t. some positive measure $\lambda$, the KL divergence is:
$KL(p\| q)\triangleq \int_{-\infty }^{\infty }p(x)\log \left(\frac{p(x)}{q(x)}\right)\,\lambda(dx)$. $\mathbb{M}$ denotes the set of communicating MDPs with a unique optimal policy. For two MDPs ${\cal M}, {\cal M}'\in \mathbb{M}$, we say that $\mcal \ll \mcal'$ if for all $(s,a)$, $p_\mcal(\cdot|s,a)\ll p_{\mcal'}(\cdot|s,a)$ and $q_\mcal(\cdot|s,a)\ll q_{\mcal'}(\cdot|s,a)$. In that case, for any state (action pair) $(s,a)$, we define the KL divergence of the distributions of the one-step observations under ${\cal M}$ and ${\cal M}'$ when starting at $(s,a)$ as $\textrm{KL}_{\mcal|\mcal'}(s,a) \triangleq KL(p_\mcal(\cdot| s,a)\| p_{\mcal'}(\cdot|s,a)) +KL(q_\mcal(\cdot|s,a)\| q_{\mcal'}(\cdot|s,a))$. 


\section{Sample complexity lower bound}\label{sec:lb}

In this section, we first derive a lower bound on the expected sample complexity satisfied by any $\delta$-PC algorithm. The lower bound is obtained as the solution of a non-convex optimization problem, as in the case where the learner has access to a generative model \cite{pmlr-v139-marjani21a}. The problem has however additional constraints, referred to as the {\it navigation} constraints, due the fact that the learner has access to a single system trajectory.\\
The expected sample complexity of an algorithm $\Alg$ is 
$
\mathbb{E}_{\mcal,\Alg}[\tau_\delta]=\sum_{s,a} \mathbb{E}_{\mcal,\Alg}[N_{sa}(\tau_\delta)].
$
Lower bounds on the sample complexity are derived by identifying constraints that the various $N_{sa}(\tau_\delta)$'s need to satisfy so as to get a $\delta$-PC algorithm. We distinguish here two kinds of constraints:   

{\bf Information constraints.} These are constraints on $\mathbb{E}_{\mcal,\Alg}[N_{sa}(\tau_\delta)]$, so that the algorithm can learn the optimal policy with probability at least $1-\delta$. They are the same as those derived \cite{pmlr-v139-marjani21a} when the learner has access to a generative model and are recalled in Lemma \ref{lem1} in Appendix \ref{sec:appendix_LB}.

{\bf Navigation constraints.} Observations come from a single (but controlled) system trajectory which imposes additional constraints on the $N_{sa}(\tau_\delta)$'s. These are derived by just writing the Chapman-Kolmogorov equations of the controlled Markov chain (refer to Appendix \ref{sec:appendix_LB} for a proof).

\begin{lemma}\label{lem2} For any algorithm $\Alg$, and for any state $s$, we have:
\begin{align}\label{eq:balance}
\Big| \mathbb{E}_{\mcal,\Alg}[N_s(\tau_\delta)]- \sum_{s',a'} p_{\mcal} (s |s',a')  \mathbb{E}_{\mcal,\Alg}[N_{s'a'}({\tau_\delta})] \Big| \le 1.
\end{align}
\end{lemma}

Putting all constraints together, we obtain the following sample complexity lower bound.


\begin{proposition}\label{prop:LB2}
Define the set of navigation-constrained allocation vectors: 
\begin{equation*}
 \Omega(\mcal) = \big\{\omega\in \Sigma :\ \forall s \in \scal,\ \omega_s = \underset{s',a'}{\sum} p_{\mcal}(s|s',a') \omega_{s' a'} \big\}.
\end{equation*}
Further define $\Alt{\mcal} = \{\mcal' \textrm{ MDP}: \mcal \ll \mcal' ,\Pi^\star(\mcal)\cap \Pi^\star(\mcal')=\emptyset\}$ the set of alternative instances. Then the expected sample complexity $\mathbb{E}_{\mcal,\Alg}[\tau_\delta]$ of any $\delta$-PC algorithm $\Alg$ satisfies:
\begin{equation}\label{eq:LBasympt}
\underset{\delta \to 0}{\liminf}\ \frac{\mathbb{E}_{\mcal,\Alg}[\tau]}{\log(1/\delta)} \geq T_o(\mcal),\quad \textrm{where}\quad T_o(\mcal)^{-1} = \sup\limits_{\omega \in \Omega(\mcal)} T(\mcal,\omega)^{-1}
\end{equation}
and
\begin{equation}
 T(\mcal,\omega)^{-1} =  \inf\limits_{\mcal' \in \mathrm{Alt}(\mcal)} \sum_{s,a}\omega_{s a}\textrm{KL}_{\mcal|\mcal'}(s,a).
\label{eq:SC_lower_bound}
\end{equation}

\end{proposition}

\begin{remark}
A common interpretation of change-of-measure lower bounds like the one above is the following: the optimization problem in the definition of $T_o(\mcal)$ can be seen as the value of a two-player zero-sum game between an algorithm which samples each state-action $(s,a)$ proportionally to $\omega_{sa}$ and an adversary who chooses an alternative instance $\mcal'$ that is difficult to distinguish from $\mcal$ under the algorithm's sampling strategy. This suggests that an optimal algorithm should play the optimal allocation that solves the optimization problem (\ref{eq:LBasympt}) and, as a consequence, rules out all alternative instances as fast as possible.
\end{remark}
\begin{remark}
Note that compared to the lower bound of the generative setting in \cite{pmlr-v139-marjani21a}, the only change is in the set of sampling strategies that the algorithm can play, which is no longer equal to the entire simplex.
\end{remark}

\subsection{Proxy for the optimal allocation and the characteristic time}
As shown in \cite{pmlr-v139-marjani21a}, even without accounting for the navigation constraints, computing the characteristic time $T_o(\mcal)$ and in particular, the optimal allocation leading to it is not easy. Indeed, the sub-problem corresponding to computing $T(\mcal,\omega)^{-1}$ is non-convex. This makes it difficult to design an algorithm that targets the optimal weights $\argmax_{\omega \in \Omega(\mcal)} T(\mcal,\omega)^{-1}$. Instead, we use a tractable upper bound of $T(\mcal,\omega)$ from \cite{pmlr-v139-marjani21a}:

\begin{lemma}(Theorem 1, \cite{pmlr-v139-marjani21a})\label{lem:proxy}
For all vectors $\omega\in \Sigma$, $T(\mcal, \omega)\leq U(\mcal,\omega)$, where\footnote{The exact definition of $H^\star$ is given in Appendix \ref{sec:appendix_LB}.}
\begin{equation}
U(\mcal,\omega) = \underset{(s,a): a\neq \pi^\star(s)}{\max}\ \frac{
H_{sa}}{\omega_{s a}} + \frac{ H^\star}{S\underset{s}{\min}\ \omega_{s,\pi^\star(s)}},
\label{eq:upper_bound}
\end{equation}
\begin{equation}
\hbox{and }
\begin{cases}
H_{sa} = \displaystyle{\frac{2}{ \Delta_{s a}^2}} + \displaystyle{\max\bigg(\frac{16\mathrm{Var}_{ p(s,a)}[\starV{\mcal}] }{\Delta_{s a}^2},\frac{6\spn{\starV{\mcal}}^{4/3}}{\Delta_{s a} ^{4/3}}\bigg)},\\
H^\star = \displaystyle{\frac{2S}{[\Delta_{\min} (1-\gamma) ]^2} + \ocal\bigg(\frac{S}{\Delta_{\min}^2(1-\gamma)^3} \bigg)}.
\end{cases}
\end{equation}
\label{lemma:upper_bound}
\end{lemma}

Using $U(\mcal,\omega)$, we obtain the following upper bound on the characteristic time (\ref{eq:LBasympt}):
\begin{equation}\label{eq:LBasympt2}
T_o(\mcal) \leq U_o(\mcal) \triangleq  \underset{{\omega \in \Omega(\mcal)}}{\inf} U(\mcal,\omega).
\end{equation}

The advantages of the above upper bound $U_o(\mcal)$ are that: (i) it is a problem-specific quantity as it depends on the gaps and variances of the value function in $\mcal$; (ii) the corresponding allocation (that solves (\ref{eq:LBasympt2})) can be easily computed, and hence targeted. Indeed, the optimization problem in (\ref{eq:LBasympt2}) has convex objective and constraints. Therefore, we can use the projected subgradient-descent algorithm to compute 
\begin{equation}
\bomega^\star(\mcal) \triangleq \underset{{\omega \in \Omega(\mcal)}}{\argmin}\ U(\mcal,\omega)=  \argmax_{\omega \in \Omega(\mcal)}\ U(\mcal,\omega)^{-1},
\label{eq:def_oracle_allocation}    
\end{equation}

which will be used in our algorithm as a proxy\footnote{Note that if we have access to an optimization oracle that, given a MDP $\mcal$, returns the optimal allocation solution to (\ref{eq:LBasympt}), then we can replace $\bomega^\star(\mcal)$ by this optimal allocation. Our algorithm will then be asymptotically optimal up to a factor of 2.} for $\argmax_{\omega \in \Omega(\mcal)}\ T(\mcal,\omega)^{-1}$.

\section{Algorithm}\label{sec:alg}
We propose MDP-NaS (MDP Navigate-and-Stop), a model-based algorithm that is inspired by the lower bound. The lower bound suggests that to identify the best policy in a sample-efficient manner, the algorithm must collect samples from state-action pair $(s,a)$ proportionally to $\bomega_{sa}^\star(\mcal)$. We propose two sampling rules which ensure that the former statement holds in the long term (see Section \ref{sec:sampling_rule} and Theorem \ref{theorem:ergodic} for a rigorous formulation). Our sampling rules are combined with a Generalized Likelihood Ratio (GLR) test (or rather a proxy of the GLR test, see Section \ref{sec:stopping_rule} for details), that stops as soon as we are confident that $\widehat{\pi}_t^\star = \pi^\star$ with probability at least $1-\delta$. The pseudo-code for MDP-NaS is given in Algorithm \ref{alg:main}.

\begin{algorithm}[h]
\label{alg:main}
\caption{MDP Navigate and Stop (MDP-NaS)}

\KwIn{Confidence level $\delta$, $\textrm{ERGODIC}$ boolean variable, communication parameter $m$ or an upper bound.}
   \eIf{\textrm{ERGODIC}}{
    Set $(\epsilon_t)_{t\geq 1} = (1/\sqrt{t})_{t\geq 1}$. \\
  }{
    Set $(\epsilon_t)_{t\geq 1} = (t^{-\frac{1}{m+1}})_{t\geq 1}$.\\
  }
   Set $t \leftarrow 0$ and $N_{sa}(t) \leftarrow 0$, for all (s,a).\\
   Initialize empirical estimate $\widehat{\mcal}_0$ by drawing an arbitrary MDP from $\mathbb{M}$.\\

  Compute $\widehat{\pi}_t^{\star}\leftarrow \textrm{POLICY-ITERATION}(\widehat{\mcal}_t)$.\\
\While{Stopping condition (\ref{eq:def_stopping_tule}) is not satisfied}
{
 Compute $\bomega^\star(\widehat{\mcal}_t)$ according to (\ref{eq:def_oracle_allocation}) and $\pi^o(\widehat{\mcal}_t)$ according to (\ref{eq:def_oracle_policy}).\\
 \eIf{$t = 0$}{
 Play $a_0 \sim \mathrm{Unif}([|1,A|])$.
 }{
 Play $a_{t} \sim \pi_t(.|s_t)$, where $\pi_t$ is determined by either (\ref{eq:def_D_navigation}) or (\ref{eq:def_C_navigation}).
 }
 Observe $(R_t,s_{t+1}) \sim q(.|s_t,a_t) \otimes p(.|s_t,a_t)$.\\
 $t \leftarrow t+1$.\\ 
 Update $\widehat{\mcal}_{t}$ and $(N_{sa}(t))_{s,a}$.\\
 Compute $\widehat{\pi}_t^{\star}\leftarrow \textrm{POLICY-ITERATION}(\widehat{\mcal}_t)$.
}
\KwOut{Empirical optimal policy $\widehat{\pi}_{\tau}^{\star}$.}   
\label{main_algorithm}
\end{algorithm}

\subsection{Sampling rule}\label{sec:sampling_rule}
We introduce a few definitions to simplify the presentation. Any stationary policy $\pi$ induces a Markov chain on $\zcal$ whose transition kernel is defined by $P_{\pi}((s,a), (s',a')) \triangleq P(s'|s,a) \pi(a'|s')$. With some abuse of notation, we will use $P_\pi$ to refer to both the Markov chain and its kernel. We denote by $\pi_u$ the uniform random policy, i.e., $\pi_u(a|s) = 1/A$ for all pairs $(s,a)$. Finally, we define the vector of visit-frequencies $\bN(t)/t \triangleq \big(N_{sa}(t)/t\big)_{(s,a)\in \zcal}$.

In contrast with pure exploration in Multi-Armed Bandits and MDPs with a generative model where any allocation vector in the simplex is achievable, here the agent can only choose a \textit{sequence of actions} and follow the resulting trajectory. Therefore, one might ask if the oracle allocation can be achieved by following a simple policy. A natural candidate is the \textit{oracle policy} defined by
\begin{align}
  \forall (s,a) \in \zcal,\quad  \pi^o(a|s) \triangleq \frac{\omega_{s a}^\star}{\omega_s^\star} = \frac{\omega_{s a}^\star}{\sum_{b\in \acal}\omega_{sb}^\star}.
\label{eq:def_oracle_policy}
\end{align}
It is immediate to check that $\bomega^\star$ is the stationary distribution of $P_{\pi^o}$. Denote by $\pi^o(\mcal)$ the above oracle policy defined through $\omega^\star(\mcal)$. Policy $\pi^o(\mcal)$ is the target that we would like to play, but since the rewards and dynamics of $\mcal$ are unknown, the actions must be chosen so as to estimate consistently the oracle policy while at the same time ensuring that $\bN(t)/t$ converges to $\bomega^\star(\mcal)$. The following two sampling rules satisfy these requirements.

{\bf D-Navigation rule:} At time step $t$, the learner plays the policy 
\begin{equation}
    \pi_t = \varepsilon_t \pi_u + (1-\varepsilon_t)\pi^o(\widehat{\mcal}_t),
\label{eq:def_D_navigation}
\end{equation} 
{\bf C-Navigation rule:} At time step $t$, the learner plays 
\begin{equation}
\pi_t = \varepsilon_t \pi_u + (1-\varepsilon_t)\sum_{j=1}^{t} \pi^o(\widehat{\mcal}_j)/t,
\label{eq:def_C_navigation}
\end{equation} 
where $\varepsilon_t$ is a decreasing exploration rate to be tuned later. In the second case, the agent navigates using a Cesàro-mean of oracle policies instead of the current estimate of the oracle policy, which makes the navigation more stable.

\subsubsection{Tuning the forced exploration parameter}
The mixture with the uniform policy\footnote{Our results still hold if $\pi_u$ is replaced by any other policy $\pi$ which has an ergodic kernel. In practice, this can be helpful especially when we have prior knowledge of a fast-mixing policy $\pi$.} helps the agent to explore all state-action pairs often enough in the initial phase. This is particularly necessary in the case where the empirical weight $\omega_{s a}^\star(\widehat{\mcal}_t)$ of some state-action pair $(s,a)$ is under-estimated, which may cause that this pair is left aside and hence the estimation never corrected. 
The next result in this section gives a tuning of the rate $\varepsilon_t$ that ensures a sufficient forced exploration:
\begin{lemma}
Let $m$ be the maximum length of the shortest paths (in terms of number of transitions) between pairs of states in $\mcal$: $
m \triangleq\max_{(s,s')\in \scal^2}\ \min \{n\geq 1:\ \exists \pi: \scal \rightarrow \acal, P_{\pi}^n(s,s') > 0  \}$. Then C-Navigation or D-Navigation with any decreasing sequence $(\varepsilon_t)_{t\geq 1}$ such that $\forall t \geq 1,\ \varepsilon_t \geq t^{-\frac{1}{m+1}}$ satisfies: $\P_{\mcal, \Alg}\big(\forall (s,a) \in \zcal,\ \lim_{t \to \infty} N_{s a}(t) = \infty \big) = 1 $.
\label{lemma:visits}
\end{lemma}

\begin{remark}
When the parameter $m$ is unknown to the learner, one can replace it by its worst case value $m_{\max} = S-1$. However, when prior knowledge is available, using a faster-decreasing sequence $\varepsilon_t = t^{-\frac{1}{m+1}}$ instead of $t^{-\frac{1}{S}}$ can be useful to accelerate convergence, especially when the states of $\mcal$ are densely connected ($m \ll S-1$).
\label{remark:dense_MDP_case}
\end{remark}

{\bf Minimal exploration rate: Communicating MDPs.} The forced exploration rate $t^{-\frac{1}{S}}$ vanishes quite slowly. One may wonder if this rate is necessary to guarantee sufficient exploration in communicating MDPs: the answer is yes in the worst case, as the following example shows. Consider a variant of the classical RiverSwim MDP with state (resp. action) space $\scal = [|1,S|]$, (resp. $\acal = \{ \small{\textsc{LEFT1, LEFT2, RIGHT}} \}$). $\small{\textsc{LEFT1}}$ and $\small{\textsc{LEFT2}}$ are equivalent actions inducing the same rewards and transitions. After playing $\small{\textsc{RIGHT}}$ the agent makes a transition of one step to the right, while playing $\small{\textsc{LEFT1}}$ (or $\small{\textsc{LEFT2}}$) moves the agent all the way back to state $1$. The rewards are null everywhere but in states $\{1,S\}$ where $R(1,\small{\textsc{LEFT1}}), R(1,\small{\textsc{LEFT2}})$ are equal to $0.01$ almost surely and $R(S,\small{\textsc{RIGHT}})$ is Bernoulli with mean 0.02. 

When $\gamma$ is close to $1$, the optimal policy consists in always playing $\small{\textsc{RIGHT}}$ so as to reach state $S$ and stay there indefinitely. Now suppose that the agent starts at $s=1$ and due to the small probability of observing the large reward in state $S$, she underestimates the value of this state in the first rounds and focuses on distinguishing the best action to play in state $1$ among $\{\small{\textsc{LEFT1, LEFT2}}\}$. Under this scenario she ends up playing a sequence of policies that scales like $\pi_t = 2\varepsilon_t \pi_u + (1-2\varepsilon_t)(\small{\textsc{LEFT1}}+\small{\textsc{LEFT2}})/2$\footnote{Modulo a re-scaling of $\epsilon_t$ by a constant factor of 1/2.}. This induces the non-homogeneous Markov Chain depicted in Figure \ref{fig:minimal_epsilon}. For the exploration rate $\displaystyle{\varepsilon_t = t^{-\alpha}}$, we show that if the agent uses any $\alpha > \frac{1}{S-1}$, with non-zero probability she will visit state $S$ only a finite number of times. Therefore she will fail to identify the optimal policy for small values of the confidence level $\delta$. The proof is deferred to Appendix \ref{sec:appendix_sampling}.
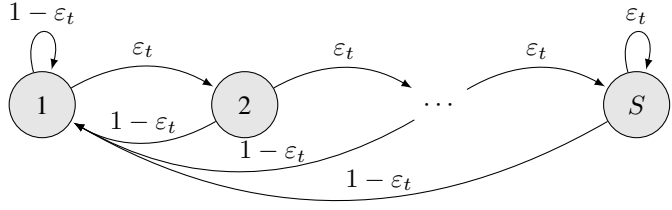
\begin{figure}[ht]
    \centering
    \begin{tikzpicture}[node distance={18mm}]
        \tikzset{node style/.style={state, fill=gray!20!white}}

        \node[node style]               (1)   {1};
        \node[node style, right=of 1]   (2)  {2};
        
        \node[draw=none,  right=of 2]   (2-S) {$\cdots$};
        \node[node style, right=of 2-S] (S)   {$S$};

    \draw[>=latex,
          auto=left,
          every loop]

         (1)   edge[bend left, auto=left] node {$\varepsilon_t$}   (2)
         (2)   edge[bend left, auto=right] node {$1-\varepsilon_t$}   (1)
         (2)  edge[bend left, auto=left] node {$\varepsilon_t$} (2-S)
         (2-S)  edge[bend left, auto=right] node {$ \quad \quad \ 1-\varepsilon_t$} (1)
         (2-S) edge[bend left, auto=left] node {$\varepsilon_t$}  (S)
         (S) edge[bend left, auto=right] node {$\quad \quad \quad 1-\varepsilon_t$}  (1)
         (1) edge[loop above] node {$1-\varepsilon_t$} (1)
         (S) edge[loop above] node {$\varepsilon_t$} (S);

\end{tikzpicture}
    \caption{Non-homogeneous Markov Chain. An exploration rate of at least $t^{-\frac{1}{S-1}}$ is needed.}
    \label{fig:minimal_epsilon}
\end{figure}

\paragraph{Minimal exploration rate: Ergodic MDPs.} For such MDPs, we can select $\varepsilon_t=1/t^\alpha$ where $\alpha<1$ without compromising the conclusion of Lemma \ref{lemma:visits}. The proof is deferred to Appendix \ref{sec:appendix_sampling}.



\subsection{Stopping rule}\label{sec:stopping_rule}
To implement a GLR test, we define $\ell_{\mcal'} (t)$, the likelihood of the observations under some MDP $\mcal' \ll \mcal$: $\ell_{\mcal'} (t) = 
\prod_{k = 0}^{t-1} 
p_{\mcal'} (s_{k+1}|s_k, a_k)
q_{\mcal'} (R_{k}|s_k, a_k)$, where at step $k$ the algorithm is in state $s_k$, plays action $a_k$ and observes the reward $R_k$ and $s_{k+1}$ (the next state). Performing a GLR test in step $t$ consists in computing the optimal policy $\widehat{\pi}_t^\star$ for the estimated MDP $\widehat{\mcal}_t$ and in comparing the likelihood of observations under the most likely model where $\widehat{\pi}_t^\star$ is optimal to the likelihood under the most likely model where $\widehat{\pi}_t^\star$ is sub-optimal. Following the standardized form of the GLR for multi-armed bandits in \cite{Degenne2019NonAsymptoticPE} we write:

\begin{align}
G_{\widehat{\pi}_t^\star} (t) &= \log\ \frac{\underset{\mcal': \widehat{\pi}_t^\star \in \Pi^\star(\mcal')}{\sup} 
\ell_{\mcal'} (t)}{\underset{\mcal': \widehat{\pi}_t^\star \notin \Pi^\star(\mcal')}{\sup} \ell_{\mcal'} (t)} \nonumber=  \log\ \frac{\ell_{\widehat{\mcal}_t}(t)}
{\underset{\mcal' \in  \Alt{\widehat{\mcal}_t}}{\sup}  \ell_{\mcal'} (t)} \nonumber
= \underset{\mcal' \in  \Alt{\widehat{\mcal}_t}}{\inf} \log \frac{\ell_{\widehat{\mcal}_t} (t)}{\ell_{\mcal'} (t)} \nonumber \\
&= \underset{\mcal' \in \Alt{\widehat{\mcal}_t}}{\inf} \underset{(s,a) \in \zcal}{\sum}\ N_{s a}(t) \big[\KL{\widehat{q}_{s,a}(t)}{q_{\mcal'}(s,a)} + \KL{\widehat{p}_{s,a}(t)}{p_{\mcal'}(s,a)}\big] \label{eq:GLR_first_exp} \\
& = t\; T\bigg(\widehat{\mcal}_t, \bN(t)/t\bigg)^{-1} \;.\label{eq:GLR_second_exp}
\end{align}

The hypothesis $(\widehat{\pi}_t^\star \neq \pi^\star)$ is then rejected as soon as the ratio of likelihoods becomes greater than the threshold $\beta(t,\delta)$, properly tuned to ensure that the algorithm is $\delta$-PC.
Note that the likelihood ratio $G_{\widehat{\pi}_t^\star} (t)$ itself can be difficult to compute for MDPs, since it is equivalent to solving (\ref{eq:SC_lower_bound}), we lower bound it using (\ref{eq:GLR_second_exp}) and Lemma \ref{lemma:upper_bound}. This leads to the following proposition:

\begin{proposition}
Define the random thresholds for the transitions and rewards respectively:
\begin{align*}
\beta_p(t,\delta) &\triangleq \log(1/\delta) + (S-1)\underset{(s,a)}{\sum}\log\bigg(e\big[1+N_{sa}(t)/(S-1)\big]\bigg)\\
\beta_r(t,\delta) &\triangleq SA\; \varphi\big(\log(1/\delta)/SA\big) + 3\sum_{s,a} \log\big[1+\log(N_{s a}(t))\big]
\end{align*}
where $\varphi(x) \underset{\infty}{\sim} x + \log(x)$ is defined in the appendix.
Then the stopping rule:
\begin{equation}
  \tau_{\delta} \triangleq \inf \bigg\{t \geq 1 :\ t\; U\big(\widehat{\mcal}_t,\bN(t)/t\big)^{-1} \geq \beta_r(t,\delta/2) + \beta_p(t,\delta/2) \bigg\}  
 \label{eq:def_stopping_tule}
\end{equation}
is $\delta$-PC, i.e., $\P(\tau_{\delta} < \infty,\ \widehat{\pi}_{\tau_\delta}^\star \neq \pi^\star) \leq \delta.$
\end{proposition}

\section{Main results and sample complexity analysis}\label{sec:analysis}
First we state our main results, which take the form of asymptotic upper bounds on the sample complexity of MDP-NaS, under a slightly more conservative exploration rate than in Lemma \ref{lemma:visits}. Then we present the most important ingredients in their proof. The complete proof is provided in Appendix \ref{sec:appendix_SC}.

\begin{theorem}
Using the C-Navigation sampling rule with $\varepsilon_t = t^{-\frac{1}{2(m+1)}}$ and the stopping rule (\ref{eq:def_stopping_tule}):
(i) MDP-NaS stops almost surely and its stopping time satisfies\\ 
$\P\big( \limsup_{\delta \to 0} \frac{\tau_\delta}{\log(1/\delta)} \leq 2 U_{o}(\mcal) \big) = 1$, where $U_{o}(\mcal)$ was defined in (\ref{eq:LBasympt2});\\
(ii) the stopping time of MDP-NaS has a finite expectation for all $\delta \in (0,1)$, and $\limsup_{\delta \to 0} \frac{\E[\tau_\delta]}{\log(1/\delta)} \leq 2 U_{o}(\mcal)$.
\label{theorem:expectation}
\end{theorem}

\subsection{Proof sketch}

{\bf Concentration of empirical MDPs:} The starting point of our proof is a concentration event of the empirical estimates $\widehat{\mcal}_t$ around $\mcal$. For $\xi > 0$ and $T\geq 1$, we define $\ccal_T^1(\xi) = \bigcap_{t= T^{1/4}}^{T}\bigg(\norm{\widehat{\mcal}_t -\mcal} \leq \xi, \norm{\pi^o(\widehat{\mcal}_t) - \pi^o(\mcal)}_\infty \leq \xi \bigg)$, where $\norm{.}$ is a semi-norm on MDPs.
Then we prove in Lemma \ref{lemma:concentration_empirical_mdps} that $\ccal_T^1(\xi)$ holds with high probability in the sense that for all $T\geq 1$:
\begin{equation}
\P\left(\ccal_T^1(\xi) \right) \geq 1 - \ocal\big(1/T^2\big).
\label{eq:HP_first_event}
\end{equation} 
For this purpose we derive, for all pairs $(s,a)$, a lower bound on $N_{sa}(t)$ stating that\footnote{Using the method in our proof one can derive a lower bound of the type $t^\zeta$ for any $\zeta <1/2$ and any exploration rate $\varepsilon_t = t^{-\theta}$ with $\theta< 1/(m+1)$. We chose $\theta = \frac{1}{2(m+1)}$ and $\zeta = 1/4$ because they enable us to have an explicit formula for $\lambda_\alpha$. See Lemma \ref{lemma:HP_forced_exploration} in the appendix.}: $\P\big(\forall (s,a) \in \zcal,\ \forall t \geq 1,\ N_{sa}(t) \geq (t/\lambda_\alpha)^{1/4} - 1\big) \geq 1 - \alpha$ where $\lambda_\alpha \propto \log^2(1+SA/\alpha)$ is a parameter that depends on the mixing properties of $\mcal$ under the uniform policy. These lower bounds on $N_{sa}(t)$ w.h.p. contrast with their deterministic equivalent obtained for C-tracking in \cite{garivier16a}. \\ \\
{\bf Concentration of visit-frequencies:} Before we proceed, we make the following assumption.
\begin{assumption}
 $P_{\pi_u}$ is aperiodic\footnote{This assumption is mild as it is enough to have only one state $\Tilde{s}$ and one action $\Tilde{a}$ such that $P_\mcal(\Tilde{s}|\Tilde{s},\Tilde{a}) > 0$ for it to be satisfied. Furthermore, Assumptions \ref{assumption:uniqueness} and \ref{assumption:aperiodic} combined are still less restrictive than the usual \textit{ergodicity} assumption, which requires that the Markov chains of \textit{all} policies are ergodic.}.
 \label{assumption:aperiodic}
\end{assumption} 
Under Assumptions \ref{assumption:uniqueness} and \ref{assumption:aperiodic} the kernel $P_{\pi_u}$, and consequently also $P_{\pi^o}$, becomes ergodic. Hence the iterates of $P_{\pi^o}$ converge to $\bomega^\star(\mcal)$ at a geometric speed. Also note that the Markov chains induced by playing either of our sampling rules are non-homogeneous, with a sequence of kernels $(P_{\pi_t})_{t\geq 1}$ that is history-dependent. To tackle this difficulty, we adapt a powerful ergodic theorem (see Proposition \ref{proposition:ergodic_thm} in the appendix) from \cite{Fort2011} originally derived for adaptive Markov Chain Monte-Carlo (MCMC) algorithms to get the following result.
\begin{theorem}
Using the C-Navigation or D-Navigation we have: $\underset{t \to \infty}{\lim} \bN(t)/t  = \omega^\star(\mcal)$ almost surely.
\label{theorem:ergodic}
\end{theorem}
Lemma \ref{lemma:visits} and Theorem \ref{theorem:ergodic} combined prove Theorem \ref{theorem:expectation} {\it (i)} in a straightforward fashion. The proof of Theorem \ref{theorem:expectation} {\it (ii)} is more involved. Again, we adapt the proof method from \cite{Fort2011} to derive a finite-time version of Proposition \ref{proposition:ergodic_thm} which results into the following proposition.
\begin{proposition}
Under C-Navigation, for all $\xi>0$, there exists a time $T_\xi$ such that for all $T \geq T_\xi$, all $t\geq T^{3/4}$ and all functions $f: \zcal \xrightarrow{} \mathbbm{R}^{+}$, we have:
\begin{align*}
  \P\bigg( \bigg|\frac{\sum_{k=1}^{t} f(s_k,a_k)}{t} - \E_{(s,a)\sim\omega^\star}[f(s,a)] \bigg| \geq K_\xi \norm{f}_\infty \xi \bigg| \ccal_T^1(\xi) \bigg) \leq 2\exp\big(-t \xi^2\big).
\end{align*}
where $\xi \mapsto K_\xi$ is a mapping with values in $(1,\infty)$ such that $\limsup_{\xi \to 0} K_\xi < \infty$.
\label{proposition:concentration_visits_informal}
\end{proposition}
Now define $\ccal_T^2(\xi) = \bigcap_{t= T^{3/4}}^{T}\big( \big|\bN(t)/t - \omega^\star(\mcal) \big| \leq K_\xi\; \xi \big)$. Then Proposition \ref{proposition:concentration_visits_informal} and Eq. (\ref{eq:HP_first_event}) combined imply that for $T$ large enough, the event $\ccal_T^1(\xi) \cap \ccal_T^2(\xi)$ holds w.h.p. so that the expected stopping time is finite on the complementary event: $\E[\tau_\delta \indicator\{\overline{\ccal_T^1(\xi)} \cup \overline{\ccal_T^2(\xi)}\}] <\infty$. Now given the asymptotic shape of the thresholds: $\beta_r(t,\delta/2) + \beta_p(t,\delta/2) \underset{\delta\to 0}{\sim} 2\log(1/\delta)$, we may informally write:
$$\E\big[\tau_\delta \indicator\{ \ccal_T^1(\xi) \cap \ccal_T^2(\xi)\}\big] \underset{\delta \to 0}{\preceq} 2\log(1/\delta) \sup_{(\mcal',\omega')\in B_\xi}  U_o\big(\mcal', \omega'\big),$$ where $B_\xi=\{(\mcal',\omega'): \norm{\mcal' -\mcal} \leq \xi, \norm{\omega' -\omega^\star(\mcal)}_\infty \leq K_\xi \; \xi\}$. Taking the limits when $\delta$ and $\xi$ go to zero respectively concludes the proof. 



\section{Comparison with MOCA}\label{sec:MOCA}
Recall from Lemma \ref{lem:proxy} and Theorem \ref{theorem:expectation} that our sample complexity bound writes as:
\begin{equation*}
    \ocal\bigg( \underset{\omega\in \Omega(\mcal)}{\inf}\ \underset{(s,a): a\neq \pi^\star(s)}{\max}\ 
\frac{1+\mathrm{Var}_{ p(s,a)}[\starV{\mcal}]}{\omega_{sa}\Delta_{s a}^2} 
+ \frac{1}{\underset{s}{\min}\omega_{s,\pi^\star(s)}\Delta_{\min}^2(1-\gamma)^3} \bigg)\log(1/\delta).
\end{equation*}
Hence, in the asymptotic regime $\delta \to 0$, MDP-NaS finds the optimal way to balance exploration between state-action pairs proportionally to their \textit{hardness}: $(1+\mathrm{Var}_{ p(s,a)}[\starV{\mcal}])/\Delta_{sa}^2$ for sub-optimal pairs (resp. $1/\Delta_{\min}^2(1-\gamma)^3$ for optimal pairs). 

After a preprint of this work was published, \cite{wagenmaker2021noregret} proposed MOCA, an algorithm for BPI in the episodic setting. MOCA has the advantage of treating the more general case of $\epsilon$-optimal policy identification with finite-time guarantees on its sample complexity. The two papers have different and complementary objectives but one can compare with their bound for exact policy identification, i.e when $\epsilon < \epsilon^\star$\footnote{$\epsilon^\star$ is defined in their work as the threshold value for $\epsilon$ such that the only $\epsilon$-optimal policy is the best policy. However, when the objective is to find the optimal policy, it is not clear from their paper how one can determine such a value of $\varepsilon$ without prior knowledge of the ground truth instance.}, in the asymptotic regime $\delta \to 0$. In this case, by carefully inspecting the proofs of \cite{wagenmaker2021noregret}, we see that MOCA's sample complexity writes as\footnote{The $\log^3(1/\delta)$ term comes from the sample complexity of their sub-routine \textsc{FindExplorableSets}.}:
\begin{equation*}
  \ocal\bigg(\sum_{h=1}^H\ \underset{\omega \in \Omega(\mcal)}{\inf}\ \underset{s,a\neq \pi^\star(s)}{\max} \frac{H^2}{\omega_{sa}(h) \Delta_{sa}(h)^2} \bigg) \log(1/\delta) + \frac{\textrm{polylog}(S,A,H, \log(\epsilon^*))\log^3(1/\delta)}{\epsilon^*}  
\end{equation*}
where $H$ is the horizon and $\Delta_{sa}(h)$ is the sub-optimality gap of $(s,a)$ at time step $h$. We make the following remarks about the bounds above:
\begin{enumerate}
    \item MOCA only pays the cost of worst-case visitation probability multiplied by the gap of the corresponding state $\underset{s,a\neq \pi^\star(s)}{\min}\omega_{sa}\Delta_{sa}^2$. Instead, MDP-NaS pays a double worst-case cost of the smallest visitation probability multiplied by the minimum gap $\min_s \omega_{s,\pi^\star(s)}\Delta_{\min}^2$. As pointed out by \cite{wagenmaker2021noregret} the former scaling is better, especially when the state where the minimum gap is achieved is different from the one that is hardest to reach. This is however an artefact of the upper bound in Lemma \ref{lem:proxy} that MDP-NaS uses as a proxy for the characteristic time. Using a more refined bound, or an optimization oracle that solves the best-response problem, one can remove this double worst-case dependency.
    \item The sample complexity of MOCA \textit{divided by $\log(1/\delta)$ still explodes} when $\delta$ goes to zero, contrary to MDP-NaS's bound.
    \item The sample complexity of MDP-NaS is variance-sensitive for sub-optimal state-action pairs, while MOCA's bound depends on a worst-case factor of $H^2$. Indeed as the rewards are in $[0,1]$, we always have $\starV{\mcal} \leq H$ and $\mathrm{Var}_{ p(s,a)}[\starV{\mcal}] \leq H^2$ in the episodic setting.
\end{enumerate}  
We conclude this section by noting that MDP-NaS has a simple design and can easily be  implemented.


\section{Conclusion}\label{sec:conclusion}
To the best of our knowledge, this paper is the first to propose an algorithm with \textit{instance-dependent sample complexity} for Best Policy identification (BPI) in the \textit{full online} setting. Our results are encouraging as they show: 1) How the navigation constraints of online RL impact the difficulty of learning a good policy, compared to the more relaxed sampling schemes of Multi-Armed Bandits and MDPs with a generative model. 2) That, provided access to an optimization oracle that solves the information-theoretical lower bound (resp. some convex relaxation of the lower bound), asymptotic optimal (resp. near-optimal) sample complexity is still possible through adaptive control of the trajectory. This opens up exciting new research questions. First, it is intriguing to understand how the mixing times -and not just the stationary distributions- of Markov chains induced by policies impact the sample complexity of BPI in the moderate confidence regime. A second direction would be to extend our contributions to the problem of finding an $\epsilon$-optimal policy, which is of more practical interest than identifying the best policy.



\bibliographystyle{alpha}
\bibliography{references}

\newpage
\appendix
\section{Symbols}
\begin{table}[ht]
  \caption{Additional notations used in the appendix}
  \label{Notation}
  \centering
  \begin{tabular}{ c | c }
    Symbol    & Definition     \\
    \toprule
    $m$ & Maximum length of shortest paths: $\underset{(s,s')\in \scal^2}{\max}\ \min \{n\geq 1:\ \exists \pi: \scal \rightarrow \acal, P_{\pi}^n(s,s') > 0  \}$ \\ \\
    $P_{\pi_u}$ & Transition kernel of the uniform policy \\ \\
    $\omega_u$    & Stationary distribution of  $P_{\pi_u}$ \\  \\
     $r$     & $\min\{ \ell \geq 1:\ \forall (z,z') \in \zcal^2,\ P_{\pi_u}^\ell (z,z') > 0 \}$  \\ \\
     $\sigma_u$ & $\min\limits_{z,z' \in \zcal }\ \frac{ P_{\pi_u}^r(z,z')}{\omega_u(z')}$\\ \\ 
     $\eta_1$ & $\min\big\{P_{\pi_u}(z,z')\ \big| (z,z')\in \zcal^2, P_{\pi_u}(z,z') > 0  \big\}$ \\ \\
     $\eta_2$ & $\min\big\{P^{n}_{\pi_u}(z,z')\ \big| (z,z')\in \zcal^2, n \in [|1,m+1|], P^{n}_{\pi_u}(z,z') > 0  \big\}$ \\ \\
     $\eta$ & Communication parameter $\eta_1\eta_2$ \\ \\
     $\omega^\star$ & oracle weights: $\omega^\star \triangleq \argmax\limits_{\omega \in \Omega(\mcal)}\ U(\mcal, \omega)$.\\ \\ 
     $\pi^o$ & oracle policy: $\displaystyle{\pi(a|s) \triangleq \frac{\omega_{sa}^\star}{\sum\limits_{a'\in \acal} \omega_{sa'}^\star}}$. \\ \\
     $\pi_{t}^o$ & $\pi^o(\widehat{\mcal}_t)$\\ \\
    $\overline{\pi_{t}^o}$ & $\sum_{j=1}^t \pi^o(\widehat{\mcal}_j)/t$\\ \\
     $Z_{\pi^o}$ & $(I -P_{\pi^o} + \mathbbm{1}{\omega^\star} \transpose)^{-1}$ \\ \\
     $\kappa_\mcal$ & Condition number $\norm{Z_{\pi^o}}_\infty$ \\ \\
     $P_t$ & Kernel of the policy $\pi_t$ \\ \\
     $\omega_t$ & Stationary distribution of $P_t$ \\ \\
     $C_t, \rho_t$ & Constants such that $\forall n\geq 1,\ \norm{P_t^n(z_0,.) - \omega_t\transpose}_\infty \leq C_t \rho_t^n$ \\ \\
     $L_t$ & $C_t (1-\rho_t)^{-1}$ \\
     
    \bottomrule
  \end{tabular}
\end{table}

\section{Experiments}

In this section, we test our algorithm on two small examples. The first instance is an ergodic MDP, with $5$ states, $5$ actions per state and a discount factor $\gamma=0.7$. The rewards of each state-action pair come from independent Bernoulli distributions with means sampled from the uniform distribution $\mathcal{U}([0,1])$. The transitions kernels were generated following a Dirichlet distribution $\mathcal{D}(1,\ldots,1)$. The second instance is the classical RiverSwim from \cite{strehl2008analysis}, which is communicating but not ergodic. The instance we used has $5$ states and $2$ actions: $\{ \textrm{LEFT, RIGHT}\}$, with deterministic transitions and a discount factor $\gamma=0.95$. Rewards are null everywhere but in states $\{1,5\}$ where they are Bernoulli with respective means $r(1,\textrm{LEFT}) = 0.05$ and $r(1,\textrm{RIGHT}) = 1$. We fix a confidence level $\delta=0.1$, and for each of these MDPs, we run $30$ Monte-Carlo simulations of MDP-NaS with either C-Navigation or D-Navigation. Towards computational efficiency, we note that the empirical oracle policy does not change significantly after collecting one sample, therefore we only update it every $10^4$ time steps\footnote{The period was chosen so as to save computation time, and knowing that for MDPs algorithms usually require $\geq 10^6$ samples to return a reasonably good policy.}.\\ First, we seek to check whether the frequencies of state-action pair visits converge to their oracle weights, as stated in Theorem \ref{theorem:ergodic}. Figure \ref{fig:omega} shows the relative distance, in log scale, between the vector of empirical frequencies $\bN(t)/t$ and the oracle allocation $\bomega^\star$. The shaded area represents the $10\%$ and $90\%$ quantiles. We see that the relative distance steadily decreases with time, indicating that the visit-frequencies of both D-Navigation and C-Navigation converge to the oracle allocation. We also note that the D-Navigation rule  exhibits a faster convergence than the C-Navigation rule.
\begin{figure}
    \centering
    \includegraphics[width=0.45\linewidth]{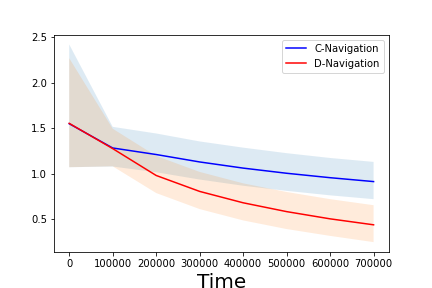}
    \includegraphics[width=0.45\linewidth]{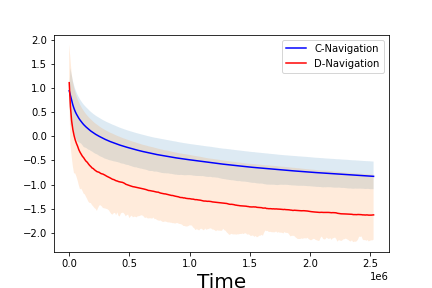}
    \caption{Relative distance in log scale: $\log_{10}\big(\max_{s,a} \frac{|N_{sa}(t)/t - \bomega^\star_{sa}|}{\bomega^\star_{sa}}\big)$. Left: Ergodic MDP with $S = A = 5$. Right: River Swim with $S=5, A=2$.}
    \label{fig:omega}
\end{figure}
Next we compare our algorithm with Variance-reduced-Q-learning (VRQL) \cite{li2020asynchronous}, a variant of the classical Q-learning with faster convergence rates. VRQL finds an $\varepsilon$-estimate $\widehat{Q}$ of the Q function $Q^\star$, by following a fixed sampling rule, referred to as the \textit{behavior policy} $\pi_b$, and updating its estimate of the Q function via Temporal Difference learning. VRQL does not have a stopping rule, but is guaranteed to yield an estimate such that $\norm{\widehat{Q} - Q}_\infty \leq \varepsilon$ with probability $1-\delta$ after using $M(N+t_{\textrm{epoch}})$ samples where
\begin{align*}
  M &= c_3 \log(1/\varepsilon^2 (1-\gamma)^2),\\
  t_{\textrm{epoch}} &= \frac{c_2}{\mu_{\min}} \bigg(\frac{1}{(1-\gamma)^3}+\frac{t_{mix}}{1-\gamma}\bigg)\log\big(1/(1-\gamma)^2 \varepsilon\big)\log\big(SA/\delta\big),\\
  N &= \frac{c_1}{\mu_{\min}} \bigg(\frac{1}{(1-\gamma)^3 \min(1,\varepsilon^2)}+ t_{mix}\bigg) \log\big(SAt_{\textrm{epoch}}/\delta\big),
\end{align*}
where $\mu_{\min}$ (resp. $t_{mix}$) is the minimum state-action occupancy (resp. the mixing time) of $\pi_b$ and $c_1, c_2, c_3$ are some large enough universal constants. We use VRQL with $c_1 = c_2 =  c_3 = 10$, $\varepsilon \leq \Delta_{\min}$ (since the goal is to identify the best policy) and use the uniform policy as a sampling rule\footnote{In the absence of prior knowledge about the MDP, the uniform policy is a reasonable choice to maximize $\mu_{\min}$.}. We plug this value into the equations above and the compute the sample complexity of VRQL. Table \ref{tab:tau} shows a comparison of the sample complexities of MDP-NaS and  VRQL. MDP-NaS has much better performance than VRQL. 
\begin{table}[h]
    \centering
    \begin{tabular}{c|c|c}
         & MDP-NaS & VRQL \\
         \toprule
        Small Ergodic MDP & $8 \times 10^5$ & $2.5\times 10^8$\\
        \midrule
         RIVER-SWIM & $2.6 \times 10^6$ & $3.3 \times 10^9$\\
         \bottomrule
    \end{tabular} \\
    \vspace{0.2cm}
    \caption{Average sample complexity of MDP-NaS (D-Navigation) vs deterministic sample complexity of VRQL. $\delta = 0.1$.}
    \label{tab:tau}
\end{table}

\section{Sample complexity lower bound}\label{sec:appendix_LB}

Let $\mathrm{Alt}(\mcal)$ be the set of MDPs such that $\mcal \ll \mcal'$ and $\Pi^\star(\mcal)\cap \Pi^\star(\mcal')=\emptyset$. The information constraints are obtained by change-of-measure arguments as in the bandit literature \cite{lai1985, kaufmann2016complexity}:  

\begin{lemma}\label{lem1}(\cite{pmlr-v139-marjani21a}) For any $\delta$-PC algorithm $\Alg$, and for any $\mcal'\in \mathrm{Alt}(\mcal)$, we have: 
\begin{equation} \label{eq:information-bdd}
\sum_{s,a}\mathbb{E}_{\mcal,\Alg}[N_{sa}(\tau_\delta)]\textrm{KL}_{\mcal|\mcal'}(s,a) \ge \kl (\delta,1-\delta).
\end{equation}
\end{lemma}

\subsection{Navigation constraints: proof of Lemma \ref{lem2}}
For all states $s$,
\begin{align*}
N_{\tau}(s) & = \indicator_{ \{ S_1=s\} } + \sum_{s',a'}\sum_{u=1}^{N_{\tau-1}(s',a')}\indicator_{ \{ W_u=s\} },
\end{align*}
where $W_u$ denotes the state observed after the $u$-th times $(s',a')$ has been visited. Fix $s',a'$. Introduce $G_t^{s',a'}= \sum_{u=1}^{N_{t-1}(s',a')}\indicator_{ \{ W_u=s\} }$. Observe that ${ \{ W_u=s\} }$ and $\{ N_{t-1}(s',a')>u-1\}$ are independent. Furthermore, $\mathbb{E}_{\mcal,\Alg}[\indicator_{ \{ W_s=s\} }]=p_{\mcal}(s|s',a')$. Hence:
$$
\mathbb{E}_{\mcal,\Alg} [G_{\tau}^{s',a'}] = p_\mcal(s|s',a')\mathbb{E}_{\mcal,\Alg}[N_{\tau-1}(s',a')].
$$
Finally,
\begin{align}
\mathbb{E}_{\mcal,\Alg} [N_\tau(s)]= \mathbb{P}_{\mcal}[\indicator_{ \{ S_1=s\} }] +\sum_{s',a'} p_\mcal(s|s',a')\mathbb{E}_{\mcal,\Alg}[N_{\tau-1}(s',a')].
\end{align}
From the above equality, the lemma is proved by just observing that $\mathbb{P}_{\mcal}[\indicator_{ \{ S_1=s\} }]\le 1$, $\mathbb{E}_{\mcal,\Alg}[N_{\tau-1}(s',a')]\le \mathbb{E}_{\mcal,\Alg}[N_{\tau}(s',a')]$ for any $(s',a')$, and $\mathbb{E}_{\mcal,\Alg}[N_{\tau}(s)]\le \mathbb{E}_{\mcal,\Alg}[N_{\tau-1}(s)]+1$ for any $s$.

\subsection{Lower bound: proof of Proposition \ref{prop:LB2}}
By combining Lemma \ref{lem1} and Lemma \ref{lem2} we get the following proposition.
\begin{proposition}\label{prop:LB1} The expected sample complexity $\mathbb{E}_{\mcal,\Alg}[\tau_\delta]$ of any $\delta$-PC algorithm $\Alg$ is larger than the value of the following optimization problem:
\begin{align}
&\inf_{n\ge 0} \ \sum_{s,a}n_{sa}\label{eq:LB1}\\
&\textrm{s.t.}\ \ \forall s, \ \Big| \sum_a n_{sa} - \sum_{s',a'} p_{\mcal} (s |s',a') n_{s'a'} \Big| \le 1,\nonumber\\
& \ \ \ \ \ \ \forall \mcal'\in \mathrm{Alt}(\mcal), \ \sum_{s,a}n_{sa}\textrm{KL}_{\mcal|\mcal'}(s,a) \ge \kl (\delta,1-\delta).\nonumber
\end{align}
\end{proposition}

In (\ref{eq:LB1}), $n_{sa}$ is interpreted as the expected number of times $(s,a)$ is visited before the stopping time. Note that the above proposition provides  a lower bound for any value of $\delta$. We can further simplify this bound when restricting our attention to asymptotic regimes where $\delta$ goes to 0. In that case, the navigation constraints (\ref{eq:balance}) can be replaced by $\sum_a n_{sa} = \sum_{s',a'} p_{\mcal} (s |s',a') n_{s'a'}$ (by just renormalizing $n$ with $n/\log(1/\delta)$, and letting $\delta\rightarrow 0$). For small $\delta$ regimes, we can hence rewrite the lower bound as follows:
\begin{align*}
\underset{\delta \to 0}{\liminf}\ \frac{\mathbb{E}_{\mcal,\Alg}[\tau]}{\log(1/\delta)} \geq\  &\inf_{n\ge 0} \ \sum_{s,a}n_{sa}\label{eq:LB1}\\
&\textrm{s.t.}\ \ \forall s, \  \sum_a n_{sa} = \sum_{s',a'} p_{\mcal} (s |s',a') n_{s'a'},\nonumber\\
& \ \ \ \ \ \ \forall \mcal'\in \mathrm{Alt}(\mcal), \ \sum_{s,a}n_{sa}\textrm{KL}_{\mcal|\mcal'}(s,a) \ge 1.\nonumber
\end{align*}
One can easily conclude by showing that the value of the optimization program above is equal to the one in Eq~\ref{eq:LBasympt}.

\subsection{Full definition of the terms in the upper bound}
Let $\mathrm{Var}_{\max}^{\star}[\starV{\mcal}] = \max\limits_{s \in \scal}\ \mathrm{Var}_{s'\sim p_{\mcal}(.|s,\pi^\star(s))}[\starV{\mcal}(s')]$ denote the maximum variance of the value function on the trajectory of the optimal policy. In \cite{pmlr-v139-marjani21a} are defined the following functionals of $\mcal$: 
\begin{align*}
T_3(\mcal) &\triangleq \frac{2}{\Delta_{\min}^2 (1-\gamma)^2},\\
T_4(\mcal) &\triangleq  \min\Bigg(\frac{27}{\Delta_{\min}^2(1-\gamma)^3},\ \max\bigg(\frac{16\mathrm{Var}_{\max}^{\star}[\starV{\mcal}]}{\Delta_{\min}^2 (1-\gamma)^2}, \frac{6\spn{\starV{\mcal}}^{4/3}}{\Delta_{\min}^{4/3} (1-\gamma)^{4/3}} \bigg)\Bigg).
\end{align*}
Then $H^\star$ is simply defined as $H^\star = S(T_3(\mcal)+T_4(\mcal))$. Note that $T_4(\mcal) = \ocal\big(\frac{1}{\Delta_{\min}^2(1-\gamma)^3} \big)$.

\section{Sampling rule}\label{sec:appendix_sampling}
Recall that $\zcal \triangleq \scal \times \acal$ denotes the set of state-action pairs. Any policy $\pi$ induces a Markov Chain on $\zcal$ whose kernel is defined by:
\begin{equation*}
    P_{\pi}((s,a), (s',a')) = P(s'|s,a) \pi(a'|s').
\end{equation*}
\paragraph{\underline{Fact 1}:} Note that if it takes $m$ steps to move between any pair of states $(s,s')$ with non-zero probability, then we can move between any pair of state-actions $((s,a), (s',a'))$ in at most $m+1$ steps by playing the policy:
\begin{align*}
    \Tilde{\pi}(x) = \begin{cases}
    a'  \quad \textrm{if $x=s'$,}\\
    \pi_{\hookrightarrow s'}(x) \quad \textrm{otherwise}
    \end{cases}
\end{align*}
where $\pi_{\hookrightarrow s'}$ is the policy corresponding to shortest path to $s'$. Finally, for the sake of simplicity, we will note $P_t \triangleq P_{\pi_t}$.

\subsection{Almost sure forced exploration: proof of Lemma \ref{lemma:visits}}
Consider the event $\ecal \triangleq \big(\exists z\in \zcal,\ \exists M > 0,\ \forall t \geq 1,\ N_{z}(t) < M \big)$. Observe that $\ecal = \bigcup\limits_{z \in \zcal} \ecal_{z}$, where for $z \in \zcal$, $\ecal_z \triangleq \big(\exists M > 0,\ \forall t \geq 1,\ N_{z}(t) < M \big)$. We will prove that $\P(\ecal_{z'}) = 0$ for all $z'$, which implies the desired result. From Fact 1 above, we have:
\begin{align}
    \forall (z,z') \in \zcal^2,\ \exists r \in [|1,m+1|],\ \exists \pi \in \Pi,\ P_{\pi}^r(z,z') > 0, 
\label{eq:communicating}
\end{align}
where $P_{\pi}^r$ is the $r$-th power of the transition matrix induced by policy $\pi$. Therefore: 
$$\eta = \underset{z,z'}{\min} \max_{\substack{1 \leq r \leq  m+1\\ \pi \in \Pi}} P_{\pi}^r(z,z') > 0$$
is well defined. Fix $z \in \zcal$ and let $\pi, r_z$ be a policy and an integer satisfying the property (\ref{eq:communicating}) above for the pair $(z,z')$. Observe that: 
\begin{align*}
P_{t} \geq \varepsilon_t P_{\pi_u} \geq \frac{\varepsilon_t}{A} P_{\pi}
\label{eq:min_coverage}
\end{align*}
where the matrix inequality is entry-wise. Now define the stopping times $(\tau_k(z))_{k \geq 1}$ where the agent reaches state-action $z$ for the $k$-th time\footnote{We restrict our attention to departure state-action pairs $z$ that are visited infinitely often. Such pairs always exist, therefore $\tau_k(z)$ is well defined.}. Then:
\begin{align*}
    \P\big(\ecal_{z'} |\ (\pi_t)_{t \geq 1},\ (\tau_k(z))_{k \geq 1} \big) &\leq \P\big(\exists N \geq 1, \forall k \geq N, X_{\tau_k(z)+r_z} \neq z'\ |\ (\pi_t)_{t \geq 1},\ (\tau_k(z))_{k \geq 1} \big)\\
    &\leq \sum_{N=1}^{\infty} \prod_{k=N}^{\infty} \P\big(X_{\tau_k(z)+r_z} \neq z'\ |\ \tau_k(z),\ (\pi_t)_{t \in [|\tau_k(z)+1,\tau_k(z)+r_z|]}  \big)\\
    &= \sum_{N=1}^{\infty} \prod_{k=N}^{\infty}\bigg[1- \bigg(\prod_{t =\tau_k(z)+1}^{\tau_k(z)+r_z}P_{\pi_t}\bigg)(z,z') \bigg]\\
    &\leq \sum_{N=1}^{\infty} \prod_{k=N}^{\infty}\bigg[1- \bigg(\prod_{t =\tau_k(z)+1}^{\tau_k(z)+r_z}\frac{\varepsilon_t}{A} P_{\pi}\bigg)(z,z')\bigg]\\
    &\leq \sum_{N=1}^{\infty} \prod_{k=N}^{\infty}\bigg[1- \frac{\eta}{A^{r_z}}\prod_{t =\tau_k(z)+1}^{\tau_k(z)+r_z}\varepsilon_t \bigg]\\
    &\leq \sum_{N=1}^{\infty} \prod_{k=N}^{\infty}\bigg[1- \frac{\eta}{A^{m+1}}\prod_{t =\tau_k(z)+1}^{\tau_k(z)+m+1}\varepsilon_t \bigg]\;.
\end{align*}
The second inequality comes from a union bound and the strong Markov property\footnote{This property is sometimes referred to as: Markov Chains start afresh after stopping times.}. The last inequality comes from the fact that $r_z \leq m+1$ and $\varepsilon_t \leq 1$. Now observe that the inequality above holds for all realizations of the sequences $(\pi_t)_{t \geq 1}$. Therefore, integrating that inequality over all possible sequences of policies yields:
\begin{align*}
 \forall z \in \zcal,\ \P\big(\ecal_{z'} |\ (\tau_k(z))_{k \geq 1} \big) \leq \sum_{N=1}^{\infty} \prod_{k=N}^{\infty}\bigg[1- \frac{\eta}{A^{m+1}}\prod_{t =\tau_k(z)+1}^{\tau_k(z)+m+1}\varepsilon_t \bigg]  \;.
\end{align*}
We can already see that if state-action $z$ is visited "frequently enough" ($\tau_k(z) \sim c. k$ for some constant $c$) then the right-hand side above will be zero. Since we know that a least one state-action $z$ is visited frequently enough, we consider the product over all state-action pairs $z$ of the probabilities above:
\begin{align}
\prod_{z \in \zcal} \P\bigg(\ecal_{z'} |\ (\tau_k(z))_{k \geq 1} \bigg) &\leq \sum_{(N_1,\ldots,N_{SA}) \in (\mathbb{N}^\star)^{SA}}\ \prod_{z \in \zcal} \prod_{k=N_z}^{\infty}\bigg[1- \frac{\eta}{A^{m+1}}\prod_{t =\tau_k(z)+1}^{\tau_k(z)+m+1}\varepsilon_t \bigg] \\
& \triangleq \sum_{(N_1,\ldots,N_{SA})} a_{(N_1,\ldots,N_{SA})} \;.\nonumber
\label{eq:ineq1}
\end{align}
We will now show that $a_{(N_1,\ldots,N_{SA})} = 0$ for all tuples $(N_1,\ldots,N_{SA})$:
\begin{align*}
 a_{(N_1,\ldots,N_{SA})} &\leq  \prod_{z \in \zcal} \prod_{k= \max\limits_{z} N_z}^{\infty}\bigg[1- \frac{\eta}{A^{m+1}}\prod_{t =\tau_k(z)+1}^{\tau_k(z)+m+1}\varepsilon_t \bigg]  \\
 & = \prod_{k= \max\limits_{z} N_z }^{\infty} \prod_{z \in \zcal} \bigg[1- \frac{\eta}{A^{m+1}}\prod_{t =\tau_k(z)+1}^{\tau_k(z)+m+1}\varepsilon_t \bigg] \;.
\end{align*}
Now observe that for all $k \geq 1$ there exists $z_k \in \zcal$ such that $\tau_k(z_k) \leq S A k$, i.e., at least one state-action has been visited $k$ times before time step $SA k$. For that particular choice of $z_k$ and since $(\varepsilon_t)_{t \geq 1}$ is decreasing, we get:
\begin{align*}
 a_{(N_1,\ldots,N_{SA})} &\leq \prod_{k= \max\limits_{z}N_{z} }^{\infty} \bigg[1- \frac{\eta}{A^{m+1}}\prod_{t =\tau_k(z_k)+1}^{\tau_k(z_k)+m+1}\varepsilon_t \bigg] \\
 & \leq \prod_{k= \max\limits_{z}N_{z} }^{\infty} \bigg[1- \frac{\eta}{A^{m+1}}\prod_{t = SA.k+1}^{SA.k+m+1}\varepsilon_t \bigg]. \\
\end{align*}
For the choice of $\varepsilon_t = t^{-\frac{1}{m+1}}$ the right-hand side above is zero. To sum up, for all realizations of $(\tau_k(z))_{z \in \zcal, k \geq 1}$:
\begin{align*}
 \prod_{z \in \zcal} \P\bigg(\ecal_{z'} |\ (\tau_k(z))_{k \geq 1} \bigg) = 0 \;.   
\end{align*}
Therefore, for all $z'$: $\P\big(\ecal_{z'}\big) = 0$ and consequently: $\P(\ecal) = 0$.

\subsection{Minimal exploration rate for communicating MDPs}
Indeed if the agent visits state $S$ at time $k$, then the last $S-1$ transitions before $k$ must have been to the right, ie $\P(s_k = S) \leq \prod_{j=k-S+1}^{k-1} \varepsilon_j \leq (\varepsilon_{k-S+1})^{S-1}$. Therefore $ \E[N_S(t)] \leq \sum_{k=S}^t (k-S+1)^{-\alpha(S-1)}$. In particular this implies that for $\alpha > \frac{1}{S-1}$, $\limsup_{t\to\infty} \E[N_S(t)] = M < \infty$. Therefore using the reverse Fatou lemma and Markov's inequality:
\begin{align*}
    \P(\forall t\geq 1, N_S(t) \leq & 2M) = \E\big[\limsup_{t\to\infty} \prod_{k=1}^t \indicator\{N_S(k)\leq 2M\}\big]\\ 
    &\geq \limsup_{t\to\infty} \E\big[ \prod_{k=1}^t \indicator\{N_S(k)\leq 2M\}\big] = \limsup_{t\to\infty} \E\big[\indicator\{N_S(t)\leq 2M\}\big] \geq \frac{1}{2}.
\end{align*}
\subsection{Minimal exploration rate for ergodic MDPs}
This is a consequence of Proposition 2 \cite{burnetas1997optimal}, stating that there exist $c_1,c_2, C>0$ such that for all $s$ and $t$ large enough, $\mathbb{P}_{\mcal, \Alg}[N_s(t)> c_1t] \ge 1-Ce^{-c_2t}$. A union bound yields: $\mathbb{P}_{\mcal, \Alg}[\forall s, N_s(t)> c_1t] \ge 1-CSe^{-c_2t}$. To extend this result to the numbers of visits at the various (state, action) pairs, we can derive a lower bound on $N_{sa}(t)$ given that $N_s(t)> c_1t$ by observing that a worst scenario (by monotonicity of $\varepsilon_s$) occurs when $s$ is visited only in the $c_1t$ rounds before $t$. We get $\mathbb{E}[N_{sa}(t) | N_s(t)> c_1t] \ge c_3t^{(1-\alpha)}$. Remarking that $N_{sa}(t+1) - N_s(t)\varepsilon_t$ is a sub-martingale with bounded increments, standard concentration arguments then imply that $\mathbb{P}_{\mcal, \Alg}[\forall s,a, N_{sa}(t)> {\frac{c_3}{2}}t^{(1-\alpha)}] \ge \phi(t)$, where $\phi(t)\to 1$. Next, define the random variable $Z_t=\prod_{s,a} \mathds{1}\{N_{sa}(t)> {\frac{c_3}{2}}t^{(1-\alpha)}\}$. Applying the reverse Fatou lemma, we get $1=\lim\sup_t\mathbb{E}[Z_t]\le \mathbb{E}[\lim\sup_t Z_t]$. From there, we directly deduce (by monotonicity of $t\mapsto N_{sa}(t)$) that a.s. $\lim_{t\to\infty} N_{sa}(t)=\infty$. 

\subsection{High probability forced exploration}
\begin{lemma}
Denote by $\tau_k(z)$ the k-th time the agent visits the state-action pair $z$. Under C-Navigation with exploration rate $\varepsilon_t = t^{-\frac{1}{2(m+1)}}$ we have: for all $\alpha \in (0,1)$, there exists a parameter $\eta > 0$ that only depends on $\mcal$ such that:
\begin{align*}
    \P\bigg(\forall z \in \zcal,\ \forall k \geq 1,\ \tau_k(z) \leq \lambda_\alpha  k^4\bigg) \geq 1 - \alpha,
\end{align*}
where $\lambda_\alpha \triangleq \frac{(m+1)^2}{\eta^2}\log^2(1+\frac{SA}{\alpha})$.
\label{lemma:HP_forced_exploration}
\end{lemma}

\begin{corollary}
Denote by $N_{z}(t)$ the number of times the agent visits state-action $z$ up to and including time step $t$. Then under the same notations of the lemma above we have: for all $\alpha \in (0,1)$:
\begin{align*}
    \P\bigg(\forall z \in \zcal,\ \forall t \geq 1,\ N_z(t) \geq \bigg(\frac{t}{\lambda_\alpha}\bigg)^{1/4} - 1\bigg) \geq 1 - \alpha.
\end{align*}
\label{corollary:forced_exploration}
\end{corollary}

\begin{proof}
Let $f$ be some increasing function such that $f(\mathbb{N}) \subset \mathbb{N}$ and $f(0) = 0$ and define the event $\ecal = \bigg(\forall z \in \zcal,\ \forall k \geq 1,\ \tau_k(z) \leq f(k) \bigg)$. We will prove the following more general result: 
\begin{align}
 \P(\ecal^c) \leq SA \sum\limits_{k=1}^{\infty} \prod\limits_{j=0}^{\floor{\frac{f(k) - f(k-1) - 1}{m+1}} - 1} \bigg[1 - \eta \prod\limits_{l=1}^{m+1} \varepsilon_{f(k-1) + (m+1).j + l} \bigg]    \;,
\end{align}
where $\eta$ is a constant depending on the communication properties of $\mcal$. Then we will tune $f(k)$ and $\varepsilon_t$ so that the right-hand side is less than $\alpha$. First, observe that:
\begin{equation*}
\ecal^c = \bigcup\limits_{z \in \zcal}\ \bigcup\limits_{k=1}^{\infty} \bigg(\tau_k(z) > f(k)\ \mathrm{and}\ \forall j \leq k-1,\ \tau_j(z) \leq f(j)  \bigg)\;. \end{equation*}
Using the decomposition above, we upper bound the probability of $\ecal^c$ by the sum of probabilities for $k \geq 1$ that the $k$-th excursion from and back to $z$ takes too long:
\begin{align}
 \P(\ecal^c) &\leq  \sum\limits_{z \in \zcal} \bigg[\P(\tau_1(z)> f(1)) + \sum\limits_{k=2}^{\infty} \P\bigg(\tau_k(z) > f(k)\ \mathrm{and}\ \forall j \leq k-1,\ \tau_j(z) \leq f(j)  \bigg) \bigg] \nonumber\\
 & \leq \sum\limits_{z \in \zcal}  \bigg[\P(\tau_1(z)> f(1)) + \sum\limits_{k=2}^{\infty} \P\bigg(\tau_k(z) > f(k)\ \mathrm{and}\  \tau_{k-1}(z) \leq f(k-1) \bigg)\bigg]\nonumber\\
 &\leq \sum\limits_{z \in \zcal} \bigg[\P(\tau_1(z)> f(1)) + \sum\limits_{k=2}^{\infty} \P\bigg(\tau_k(z) - \tau_{k-1}(z)> f(k) - f(k-1),\ \tau_{k-1}(z) \leq f(k-1) \bigg)\bigg] \nonumber\\
 &\leq \sum\limits_{z \in \zcal}\bigg[\P(\tau_1(z)> f(1))+\sum\limits_{k=2}^{\infty} \sum\limits_{n =1}^{f(k-1)} \P\big(\tau_k(z) - \tau_{k-1}(z)> f(k) - f(k-1) \big| \tau_{k-1}(z) = n \big) \P(\tau_{0}(z) = n) \bigg] \nonumber \\ 
 & = \sum\limits_{z \in \zcal} \big[a_{1}(z) + \sum\limits_{k=2}^{\infty}\sum\limits_{n=1}^{f(k-1)} a_{k,n}(z) \P(\tau_{k-1}(s) = n) \big]\;,
\label{eq:4}
\end{align}
where
\begin{align*}
 a_{1}(z) &\triangleq \P(\tau_1(z)> f(1))\;, \\
 \forall k \geq 2 \ \forall n\in [|1,f(k-1)|],\ a_{k,n}(z) &\triangleq \P\big(\tau_k(z) - \tau_{k-1}(z)> f(k) - f(k-1) \big| \tau_{k-1}(z) = n \big)\;.
\end{align*}
We will now prove an upper bound on  $a_{k,n}(z)$ for a fixed $z \in \zcal$ and $k \geq 1$.
\paragraph{1) \underline{Upper bounding the probability that an excursion takes too long:}} Let us rewrite $P_t$ as 
$$P_t = 
\begin{pmatrix}
  \begin{matrix}
  \\
  \quad Q_t(z) \quad \\
  \\
  \end{matrix}
  & \rvline & [P_t(z',z)]_{z' \neq z} \\
\hline
   [P_t(z, z')]_{z' \neq z}^{T} & \rvline &
  \begin{matrix}
    P_t(z,z)
  \end{matrix}
\end{pmatrix}\;,
$$ 
so that state-action $z$ corresponds to the last row and last column. Further let $p_t(z', \neg z) \triangleq [P_t(z', z")]_{z" \neq z}$ denote the vector of probabilities of transitions at time $t$ from $z'$ to  states $z"$ different from $z$. Using a simple recurrence on $N$, one can prove that for all $k,N,n \geq 1$ we have:
\begin{align}
\P\bigg(\tau_k(z) - \tau_{k-1}(z)> N \bigg| \tau_{k-1}(z) = n \bigg) = p_n(z, \neg z)\transpose \bigg( \prod\limits_{j=n+1}^{n+N-1} Q_j(z) \bigg) \mathbbm{1} \;.
\label{eq:Recurrence_excursion}
\end{align}
Using Lemma \ref{lemma:substochastic_matrices}, there exists $\eta > 0$ (that only depends on $\mcal$) such that for all $n \geq 1$ and all sequences $(\pi_t)_{t\geq 1}$ such that $\pi_t \geq \varepsilon_t \pi_u$ we have: 
\begin{equation}
\norm{\prod\limits_{l=n+1}^{n+m+1} Q_l(z)}_\infty \leq 1 - \eta \prod\limits_{l=n+1}^{n+m+1} \varepsilon_l\;. 
\label{eq:lemma_Q}
\end{equation}
Therefore using (\ref{eq:Recurrence_excursion}) for $N = f(k) - f(k-1)$ and breaking the matrix product into smaller product terms of $(m+1)$ matrices, we get for $k \geq 2$:
\begin{align}
a_{k,n}(z) &= \P\bigg(\tau_k(z) - \tau_{k-1}(z)> f(k) - f(k-1) \bigg| \tau_{k-1}(z) = n \bigg) \nonumber\\
    &= \E_{(\pi_t)_{t\geq 1}}\bigg[\P\bigg(\tau_k(s) - \tau_{k-1}(s)> f(k) - f(k-1) \bigg| \tau_{k-1}(z) = n, (\pi_t)_{t\geq 1} \bigg)\bigg] \nonumber\\
    &= \E_{(\pi_t)_{t\geq 1}}\bigg[p_n(z, \neg z)\transpose \bigg( \prod\limits_{j=n+1}^{n+f(k)-f(k-1)-1} Q_j(z) \bigg) \mathbbm{1}\bigg]\nonumber \\
    &\leq \norm{\prod\limits_{l=n+1}^{n+f(k)-f(k-1)-1} Q_l(z)}_\infty\nonumber \\
    &\leq \norm{\prod\limits_{l=(m+1)\floor{\frac{f(k) - f(k-1) - 1}{m+1}}+1}^{f(k) - f(k-1) - 1} Q_{n+l}(z)}_\infty \times \prod\limits_{j=0}^{\floor{\frac{f(k) - f(k-1) - 1}{m+1}} - 1} \norm{\prod_{l=1}^{m+1} Q_{n+(m+1)j+l}(z)}_{\infty}\nonumber\\
    &\leq \prod\limits_{j=0}^{\floor{\frac{f(k) - f(k-1) - 1}{m+1}} - 1} \bigg[1 - \eta \prod_{l=1}^{m+1} \varepsilon_{n+(m+1)j+l}\bigg]\nonumber\\
    &\leq \prod\limits_{j=0}^{\floor{\frac{f(k) - f(k-1) - 1}{m+1}} - 1} \bigg[1 - \eta \prod_{l=1}^{m+1} \varepsilon_{f(k-1)+(m+1)j+l}\bigg] \triangleq b_k\;,
\label{eq:5}
\end{align}
where in the fourth line we used that $\norm{p_n(z, \neg z)}_1 \leq 1$. The sixth line uses the fact that the matrices $Q$ are substochastic. The last line is due to the fact that  $n \leq f(k-1)$ and $t \mapsto \varepsilon_t$ is decreasing. Similarly, one can prove that:
\begin{align}
    a_1(z) &\leq  \prod\limits_{j=0}^{\floor{\frac{f(1)- 1}{m+1}} - 1} \bigg[1 - \eta \prod_{l=1}^{m+1} \varepsilon_{(m+1)j+l}\bigg] \nonumber\\
    &= \prod\limits_{j=0}^{\floor{\frac{f(1) - f(0) - 1}{m+1}} - 1} \bigg[1 - \eta \prod_{l=1}^{m+1} \varepsilon_{f(0)+(m+1)j+l}\bigg] \triangleq b_1\;,
\label{eq:5.5}
\end{align}
where we used the fact that $f(0) = 0$. Now we only have to tune $f(k)$ and $\varepsilon_t$ so that $\sum\limits_{k=1}^\infty b_k  < \frac{\alpha}{SA}$ and conclude using (\ref{eq:4}), (\ref{eq:5}) and (\ref{eq:5.5}).

\paragraph{2) \underline{Tuning $f$ and the exploration rate:}} Since the sequence $(\varepsilon_t)_{t\geq 1}$ is decreasing we have:
\begin{align*}
    b_k &= \prod\limits_{j=0}^{\floor{\frac{f(k) - f(k-1) - 1}{m+1}} - 1} \bigg[1 - \eta \prod_{l=1}^{m+1} \varepsilon_{f(k-1)+(m+1)j+l}\bigg] \\
    &\leq \prod\limits_{j=0}^{\floor{\frac{f(k) - f(k-1) - 1}{m+1}} - 1} \bigg[1 - \eta  \big(\varepsilon_{f(k-1)+(m+1)j+S}\big)^{m+1}\bigg] \\
    &\leq \bigg[1 - \eta  \big(\varepsilon_{f(k)}\big)^{m+1} \bigg]^{\floor{\frac{f(k) - f(k-1) - 1}{m+1}}}.
\end{align*}
For $f(k) = \lambda.k^4$ where $\lambda \in \mathbb{N}^\star$ and $\varepsilon_t = t^{-\frac{1}{2(m+1)}}$ we have: $\floor{\frac{f(k) - f(k-1) - 1}{m+1}} \geq \frac{\lambda k^3}{(m+1)}$ and $\big(\varepsilon_{f(k)}\big)^{m+1} = \frac{1}{\sqrt{\lambda} k^2}$, implying:
\begin{align*}
b_k &\leq \bigg[1 - \frac{\eta}{\sqrt{\lambda} k^2}\bigg]^{\frac{\lambda k^3}{(m+1)}}\\
&\leq \exp\bigg(\frac{-\lambda k^3 \eta }{(m+1)\sqrt{\lambda} k^2}\bigg) = \exp\bigg(-\frac{\lambda^{1/2}k \eta}{m+1}\bigg) \;.   
\end{align*}
Summing the last inequality, along with (\ref{eq:4}), (\ref{eq:5}) and (\ref{eq:5.5}) we get:
\begin{align*}
    \P(\ecal^c) &\leq SA \sum\limits_{k=1}^{\infty} b_k\\
    &\leq SA \sum\limits_{k=1}^{\infty}\exp\bigg(-\frac{\lambda^{1/2}k\eta}{m+1}\bigg) \\
    &= \frac{SA \exp\big(-\frac{\lambda^{1/2}\eta}{m+1}\big)}{1 - \exp\big(-\frac{\lambda^{1/2} \eta}{m+1}\big)} \triangleq g(\lambda)\;.
\end{align*}
For $\lambda_\alpha \triangleq \frac{(m+1)^2}{\eta^2}\log^2(1+\frac{SA}{\alpha})$, we have $g(\lambda_\alpha) = \alpha$, which gives the desired result. 
\begin{remark}
It is natural that $\lambda$ depends on $\eta$, which expresses how well connected is the MDP under the uniform policy, see proof of Lemma \ref{lemma:substochastic_matrices}.
\end{remark}
\end{proof}

\subsection{An Ergodic Theorem for non-homogeneous Markov Chains}
We start with some definitions and a technical result. Let $\{P_{\pi}, \pi \in \Pi \}$ be a collection of Markov transition kernels on the state-action space $\zcal$, indexed by policies $\pi \in \Pi$. For any Markov transition kernel $P$, bounded function $f$ and probability distribution $\mu$, we define:
\begin{equation*}
    Pf(z) \triangleq \sum\limits_{z'\in \zcal} P(z,z')f(z')\quad \textrm{and}\quad \mu P(z) \triangleq \sum\limits_{z'\in \zcal} \mu(z')P(z',z).
\end{equation*}
For a measure $\mu$ and a function $f$, $\mu(f) = \E_{X\sim\mu}[f(X)]$ denotes the mean of $f$ w.r.t. $\mu$. Finally, for two policies $\pi$ and $\pi'$ we define $D(\pi,\pi') \triangleq \norm{P_{\pi} - P_{\pi'}}_\infty = \max\limits_{z\in\zcal}\norm{P_{\pi}(z,.) - P_{\pi'}(z,.)}_1$.\\
We consider a $\zcal\times\Pi$-valued process $\{(z_t, \pi_t), t\geq1\}$ such that $(z_t,\pi_t)$ is $\fcal_t$-adapted and for any bounded measurable function $f$:
$$ 
\E[f(z_{t+1})|\fcal_t] = P_{\pi_t} f(z_t).
$$
The next result is adapted from \cite{Fort2011}. There the authors prove an ergodic theorem for adaptive Markov Chain Monte-Carlo (MCMC) algorithms with a general state space and a parameter-dependent function. For the sake of self-containedness, we include here the proof of their result in the simple case of finite state space chains with a function that does not depend on the policy $\pi_t$. 
\begin{proposition}\label{prop:ergodic}(Corollary 2.9, \cite{Fort2011})
Assume that:
\begin{align*}
    &\textrm{\hypertarget{assumption:B1}{(B1)}}\ \forall t\geq1,\ P_t\ \textrm{is ergodic. We denote by $\omega_t$ its stationary distribution.}\\
    &\textrm{\hypertarget{assumption:B2}{(B2)}}\ \textrm{There exists an ergodic kernel $P$ such that $\norm{P_t - P}_\infty \underset{t \to\infty}{\longrightarrow} 0$ almost surely.} \\
    &\textrm{\hypertarget{assumption:B3}{(B3)} There exists two constants $C_t$ and $\rho_t$ such that for all $n\geq1$, } \norm{P_t^n - W_t}_\infty \leq C_t \rho_t^n\ ,\\
    &\quad \quad \textrm{where $W_t$ is a rank-one matrix whose rows are equal to $\omega_t\transpose$}.\\
    &\textrm{\hypertarget{assumption:B4}{(B4)}}\ \textrm{Denote by } L_t \triangleq C_t (1-\rho_t)^{-1}. \textrm{ Then } \limsup\limits_{t\to\infty} L_t < \infty \textrm{ almost surely}.\textsc{ (UNIFORM ERGODICITY)}\\
    &\textrm{\hypertarget{assumption:B5}{(B5)}}\  D(\pi_{t+1},\pi_t) \underset{t \to\infty}{\longrightarrow} 0 \textrm{ almost surely}. \textsc{ (STABILITY)}
\end{align*}
Finally, denote by $\omega^\star$ the stationary distribution of $P$. Then for any bounded non-negative function $f:\ \zcal \to \mathbb{R}^{+}$ we have:
\begin{align*}
    \frac{\sum\limits_{k=1}^{t} f(z_k)}{t} \underset{t \to\infty}{\longrightarrow}  \omega^\star(f)
\end{align*}
almost surely.
\label{proposition:ergodic_thm}
\end{proposition}

\begin{proof}
Consider the difference
\begin{align}
    D &= \frac{\sum\limits_{k=1}^{t} f(z_k)}{t} - \omega(f) \nonumber\\
    &= \underbrace{\frac{f(z_1) - \omega^\star(f)}{t}}\limits_{D_{1,t}} + \underbrace{\frac{\sum\limits_{k=2}^{t} \big[f(z_k) - \omega_{k-1}(f)\big]}{t}}\limits_{D_{2,t}} + \underbrace{\frac{\sum\limits_{k=2}^{t} \big[\omega_{k-1}(f) - \omega^\star(f)\big]}{t}}\limits_{D_{3,t}}.\\
\label{eq:diff_decomposition}
\end{align}
We clearly have: 
\begin{align}
|D_{1,t}| \leq \frac{\norm{f}_\infty}{t} \underset{t \to\infty}{\longrightarrow} 0.
\label{ineq:D1}
\end{align}
Next, by Lemma \ref{lemma:Schweitzer} there exists a constant $\kappa_P$ (that only depends on $P$) such that: $\norm{\omega_k - \omega^\star}_1 \leq \kappa_P \norm{P_k - P}_\infty$. Therefore:
\begin{align}
|D_{3,t}| &\leq \kappa_P \norm{f}_\infty \frac{\sum\limits_{k=1}^{t-1} \norm{P_{k} - P}_\infty }{t}\underset{t \to\infty}{\longrightarrow} 0,
\label{ineq:D3} 
\end{align}
where the convergence to zero is due to assumption \hyperlink{assumption:B2}{(B2)}. Now to bound $D_{2,t}$ we use the function $\widehat{f}_k$ solution to the Poisson equation $\big(\widehat{f}_k - P_k\widehat{f}_k\big)(.) = f(.) - \omega_{k}(f)$. By Lemma \ref{lemma:poisson}, $\widehat{f}_k(.) = \sum\limits_{n\geq0} P_k^n[f - \omega_{k}(f)](.)$ exists and is solution to the Poisson equation. Therefore we can rewrite $D_{2,t}$ as follows:
\begin{align}
    D_{2,t} &= \frac{\sum\limits_{k=2}^{t} \big[\widehat{f}_{k-1}(z_k) - P_{k-1}\widehat{f}_{k-1}(z_k)\big]}{t} \nonumber\\
    &= M_t + C_t + R_t,
\label{eq:D_2_decomposition}
\end{align}
where
\begin{align*}
& M_t \triangleq \frac{\sum\limits_{k=2}^{t} \big[\widehat{f}_{k-1}(z_k) - P_{k-1}\widehat{f}_{k-1}(z_{k-1})\big]}{t},\\ 
& C_t \triangleq  \frac{\sum\limits_{k=2}^{t} \big[P_{k}\widehat{f}_{k}(z_k) - P_{k-1}\widehat{f}_{k-1}(z_k)\big]}{t},\\
& R_t \triangleq  \frac{P_1\widehat{f}_{1}(z_1) - P_{t}\widehat{f}_{t}(z_{t})}{t}.
\end{align*}
\paragraph{Bounding $M_t$: }Note that $S_t \triangleq t M_t$ is a martingale since $\E[\widehat{f}_{k-1}(z_k) | \fcal_{k-1}] = P_{k-1}\widehat{f}_{k-1}(z_{k-1})$. Furthermore, by Lemma \ref{lemma:poisson}:
\begin{align*}
    |S_t - S_{t-1}| &= |\widehat{f}_{t-1}(z_t) - P_{k-1}\widehat{f}_{t-1}(z_{t-1})|\\
    &\leq 2\norm{\widehat{f}_{t-1}}_\infty\\
    &\leq 2 L_{t-1} \norm{f}_\infty.
\end{align*}
In particular, this implies that:
\begin{align*}
    \sum\limits_{k=2}^{\infty} \frac{\E[ |S_k - S_{k-1}|^2\ | \fcal_{k-1}]}{k^2} &\leq \sum\limits_{k=2}^{\infty} \frac{4\norm{f}_\infty^2 L_{k-1}^2}{k^2} < \infty
\end{align*}
where the convergence of the series is due to \hyperlink{assumption:B4}{(B4)}.
(Theorem 2.18 in \cite{HALL80}) then implies that $M_t \underset{t \to\infty}{\longrightarrow} 0$ almost surely.
\paragraph{Bounding $C_t$: } Using Lemma \ref{lemma:poisson}, we have:
\begin{align}
|C_t| &\leq \norm{f}_\infty \frac{\sum\limits_{k=2}^{t} L_k\bigg[ \norm{\omega_k - \omega_{k-1}}_1 + L_{k-1} D(\pi_k, \pi_{k-1}) \bigg] }{t} \nonumber\\
&\leq \norm{f}_\infty \frac{\sum\limits_{k=2}^{t} L_k\bigg[  \norm{\omega_k - \omega^\star}_1 + \norm{\omega^\star - \omega_{k-1}}_1 + L_{k-1} D(\pi_k, \pi_{k-1}) \bigg] }{t} \nonumber \\
&\leq \norm{f}_\infty \bigg[ \kappa_P \frac{L_2 \norm{P_1 - P}_\infty + L_t \norm{P_{t} - P}_\infty +\sum\limits_{k=2}^{t-1} (L_k + L_{k+1})\norm{P_{k} - P}_\infty}{t}\nonumber \\
& \quad \quad \quad \quad +  \frac{\sum\limits_{k=2}^{t} L_k L_{k-1} D(\pi_k, \pi_{k-1})}{t} \bigg] \underset{t \to\infty}{\longrightarrow} 0,
\label{ineq:C_t}    
\end{align}
where the third line comes from applying Lemma \ref{lemma:Schweitzer} and the convergence to zero is due to assumptions \hyperlink{assumption:B2}{(B2)}-\hyperlink{assumption:B4}{(B4)}-\hyperlink{assumption:B5}{(B5)}.
\paragraph{Bounding $R_t$: } Finally, by Lemma \ref{lemma:poisson} we have:
\begin{align}
    |R_t| &\leq \frac{\norm{\widehat{f}_{1}}_\infty + \norm{\widehat{f}_t}_\infty}{t} \nonumber\\
        &\leq  \frac{\norm{f}_\infty (L_1 + L_t)}{t}\underset{t \to\infty}{\longrightarrow} 0,
\label{ineq:R_t}
\end{align}
where the convergence to zero is due to Assumption \hyperlink{assumption:B4}{(B4)}. Summing up the inequalities (\ref{ineq:D1}-\ref{ineq:R_t}) gives the result.
\end{proof}

\subsection{Application to C-Navigation: proof of Theorem \ref{theorem:ergodic}}
\paragraph{We will now prove that C-Navigation verifies the assumptions (B1-5).} The same can be proved for D-Navigation by replacing $\overline{\pi_t^o}$ with $\pi_t^o$. Theorem \ref{theorem:ergodic} follows immediately by applying Proposition \ref{proposition:ergodic_thm} for the functions $f(\Tilde{z}) = \mathbbm{1}\{\Tilde{z}=z\}$, where $z$ is any fixed state-action pair.

\paragraph{\underline{(B1)}:} This is a direct consequence of the fact that $P_{\pi_u}$ is ergodic (due to Assumptions \ref{assumption:uniqueness} and \ref{assumption:aperiodic}) which implies by construction that $P_t$ is also ergodic.

\paragraph{\underline{(B2)}:} By Lemma \ref{lemma:visits} we have: $N_{sa}(t) \overset{a.s}{\longrightarrow} \infty$ for all $(s,a)$. Hence $\widehat{\mcal} \overset{a.s}{\longrightarrow} \mcal$ and by continuity: $\pi^o(\widehat{\mcal}) \overset{a.s}{\longrightarrow} \pi^o(\mcal)$, which implies that:
\begin{equation}
P_t \overset{a.s}{\longrightarrow} P_{\pi^o}.
\label{eq:kernel_cvg}
\end{equation}

\paragraph{\underline{(B3)}:} By Lemma \ref{lemma:Geometric_C_Navigation}, $P_t$ satisfies \hyperlink{assumption:B3}{(B3)} for $C_t = 2\theta(\varepsilon_t, \overline{\pi_t^o}, \omega_t)^{-1}$ and $\rho_t = \theta(\varepsilon_t, \overline{\pi_t^o}, \omega_t)^{1/r}$.

\paragraph{\underline{(B4)}:} We have:
\begin{align}
  \sigma(\varepsilon_t, \overline{\pi_t^o}, \omega_t) &=  \bigg(\varepsilon_t^r + \big[(1-\varepsilon_t) A \min\limits_{s,a} \overline{\pi_t^o}(a|s) \big]^r \bigg) \sigma_u \bigg(\min\limits_{z}\frac{\omega_u(z)}{\omega_t(z)}\bigg) \nonumber\\
  &\overset{a.s}{\longrightarrow} \bigg(A \min\limits_{s,a} \pi^o(a|s) \bigg)\sigma_u \min\limits_{z}\frac{\omega_u(z)}{\omega^\star(z)} \triangleq \sigma_o.
\label{eq:sigma_o}
\end{align}
Note that $\sigma_o > 0$ since $\omega_u > 0$ (ergodicity of $P_{\pi_u}$), $\omega^\star < 1$ and $\pi^o > 0$ entry-wise. Similarly, it is trivial that $\sigma_o < 1$ since $A \min\limits_{s,a} \pi^o(a|s) < 1, \min\limits_{z}\frac{\omega_u(z)}{\omega^\star(z)} < 1$ and $\sigma_u < 1$.  Therefore $\theta(\varepsilon_t, \overline{\pi_t^o}, \omega_t)= 1 - \sigma(\varepsilon_t, \overline{\pi_t^o}, \omega_t) \overset{a.s}{\longrightarrow} 1 - \sigma_o \triangleq \theta_o  \in (0,1)$ and:
\begin{align}
   \limsup\limits_{t\to\infty}\ L_t &= \limsup\limits_{t\to\infty}\ C_t (1-\rho_t)^{-1}\nonumber \\
   &= \limsup\limits_{t\to\infty} \frac{2}{\theta(\varepsilon_t, \overline{\pi_t^o}, \omega_t) \big[1- \theta(\varepsilon_t, \overline{\pi_t^o}, \omega_t)^{1/r}\big]}\nonumber\\
   &= \frac{2}{\theta_o \big[1-\theta_o^\frac{1}{r}\big]} < \infty.
\label{eq:theta_o}
\end{align}

\paragraph{\underline{(B5)}:} We have: $\big(P_{t+1} - P_t\big)(s,s') = \sum\limits_{a} [\pi_{t+1}(a|s) - \pi_t(a|s)] p_\mcal(s'|s,a)$. Hence: $\norm{P_{t+1} - P_t}_\infty \leq \norm{\pi_{t+1} - \pi_t}_\infty$, where $\pi_{t+1}$ and $\pi_t$ are viewed as vectors. On the other hand:
\begin{align*}
    \pi_{t+1} - \pi_t &= (\varepsilon_t - \varepsilon_{t+1})(\overline{\pi_t^o} - \pi_u) + (1-\varepsilon_t)(\overline{\pi_{t+1}^o} - \overline{\pi_t^o})\\
    &= (\varepsilon_t - \varepsilon_{t+1})(\overline{\pi_t^o} - \pi_u) + (1-\varepsilon_t)(\frac{t\times \overline{\pi_{t}^o} + \pi^o(\widehat{\mcal}_{t+1})}{t+1} - \overline{\pi_t^o})\\
    &= (\varepsilon_t - \varepsilon_{t+1})(\overline{\pi_t^o} - \pi_u) + (1-\varepsilon_t)\frac{ \pi^o(\widehat{\mcal}_{t+1}) - \overline{\pi_{t}^o}}{t+1}
\end{align*}
Therefore
\begin{align*}
  \norm{P_{t+1} - P_t}_\infty &\leq \norm{\pi_{t+1} - \pi_t}_\infty\\
  &\leq (\varepsilon_t - \varepsilon_{t+1}) + \frac{1}{t+1} \underset{t \to\infty}{\longrightarrow} 0.
\end{align*}
For D-Navigation, we get in a similar fashion:
\begin{align*}
  \norm{P_{t+1} - P_t}_\infty &\leq \norm{\pi_{t+1} - \pi_t}_\infty\\
  &\leq (\varepsilon_t - \varepsilon_{t+1}) + (1-\varepsilon_t)\norm{\pi_{t+1}^o - \pi_t^o}_\infty \underset{t \to\infty}{\longrightarrow} 0.
\end{align*}

\subsection{Geometric ergodicity of the sampling rules}
Since $P_{\pi_u}$ is ergodic, there exists $r>0$ such that $P_{\pi_u}^r(z,z')>0$ for all $z,z'$ (Proposition 1.7, \cite{LevinPeresWilmer2006}). Thus we define
\begin{align}
     r &= \min\{ \ell \geq 1:\ \forall (z,z') \in \zcal^2,\ P_{\pi_u}^\ell (z,z') > 0 \},\\ 
   \sigma_u &\triangleq \min\limits_{z,z'}\ \frac{ P_{\pi_u}^r(z,z')}{\omega_u(z')},   
\label{eq:pi_u_ergodic}
\end{align}
where $\omega_u$ is the stationary distribution of $P_{\pi_u}$.
\begin{lemma}
Let  $\pi_t^o \triangleq \pi^o(\widehat{\mcal}_t)$ (resp. $\overline{\pi_t^o} \triangleq \sum_{j=1}^{t} \pi^o(\widehat{\mcal}_j)/t$) denote the oracle policy of $\widehat{\mcal}_t$ (resp. the Cesaro-mean of oracle policies up to time $t$). Further define
\begin{align*}
    \sigma(\varepsilon,\pi,\omega) &\triangleq \bigg(\varepsilon^r + \big[(1-\varepsilon) A \min\limits_{s,a} \pi(a|s) \big]^r \bigg) \sigma_u \bigg(\min\limits_{z}\frac{\omega_u(z)}{\omega(z)}\bigg),\\
    \theta(\varepsilon,\pi,\omega) &\triangleq 1 - \sigma(\varepsilon,\pi,\omega),\\
    \lcal(\varepsilon, \pi, \omega) &\triangleq \frac{2}{\theta(\varepsilon,\pi,\omega) \big[1- \theta(\varepsilon,\pi,\omega)^{1/r}\big]}.
\end{align*}

Then for D-Navigation (resp. C-Navigation) we have:
\begin{align*}
\forall n\geq 1,\  \norm{P_t^n - W_t}_\infty \leq C_t \rho_t^n   
\end{align*}
where $C_t = 2\theta(\varepsilon_t, \pi_t^o, \omega_t)^{-1}$ and $\rho_t = \theta(\varepsilon_t, \pi_t^o, \omega_t)^{1/r}$ \big(resp. $C_t = 2\theta(\varepsilon_t, \overline{\pi_t^o}, \omega_t)^{-1}$ and $\rho_t = \theta(\varepsilon_t, \overline{\pi_t^o}, \omega_t)^{1/r}$\big). In particular $L_t \triangleq C_t (1-\rho_t)^{-1} = \lcal(\varepsilon_t, \pi_t^o, \omega_t)$ \big(resp. $L_t = C_t (1-\rho_t)^{-1} = \lcal(\varepsilon_t, \overline{\pi_t^o}, \omega_t)$\big).
\label{lemma:Geometric_C_Navigation}
\end{lemma}
\begin{proof}
We only prove the lemma for C-Navigation. The statement for D-Navigation can be proved in the same way. Recall that: $P_t = \varepsilon_t P_{\pi_u} + (1-\varepsilon_t) P_{\overline{\pi_t^o}}$. Therefore:

\begin{align*}
\forall (z,z'),\quad    P_t^r(z,z') &\geq  [\varepsilon_t^r P_{\pi_u}^r + (1-\varepsilon_t)^r P_{\overline{\pi_t^o}}^r](z,z')\\
    &\geq  \bigg(\varepsilon_t^r + \big[(1-\varepsilon_t) A \min\limits_{s,a} \overline{\pi_t^o}(a|s) \big]^r \bigg) P_{\pi_u}^r(z,z')\\
    &\geq \bigg(\varepsilon_t^r + \big[(1-\varepsilon_t) A \min\limits_{s,a} \overline{\pi_t^o}(a|s) \big]^r \bigg) \sigma_u \omega_u(z') \\
    & \geq \underbrace{\bigg(\varepsilon_t^r + \big[(1-\varepsilon_t) A \min\limits_{s,a} \overline{\pi_t^o}(a|s) \big]^r \bigg) \sigma_u \bigg(\min\limits_{z}\frac{\omega_u(z)}{\omega_t(z)}\bigg)}\limits_{\sigma_t} \omega_t(z')\\
    &= \sigma(\varepsilon_t, \overline{\pi_t^o}, \omega_t) \omega_t(z').
\end{align*}
where the second and third inequalities comes from the fact that $P_{\overline{\pi_t^o}} \geq A \min\limits_{s,a} \overline{\pi_t^o}(a|s) P_{\pi_u}$ entry-wise and from (\ref{eq:pi_u_ergodic}) respectively.
Using Lemma \ref{lemma:geometric_ergodicity} we conclude that or all $n\geq1$:
\begin{equation*}
    \norm{P_t^n - W_t}_\infty \leq 2 \theta(\varepsilon_t, \overline{\pi_t^o}, \omega_t)^{\frac{n}{r}-1}
\end{equation*}
where $\theta(\varepsilon_t, \overline{\pi_t^o}, \omega_t) = 1 - \sigma(\varepsilon_t, \overline{\pi_t^o}, \omega_t)$. Therefore $P_t$ satisfies $\norm{P_t^n - W_t}_\infty \leq C_t \rho_t^n $ for $C_t = 2\theta(\varepsilon_t, \overline{\pi_t^o}, \omega_t)^{-1}$ and $\rho_t = \theta(\varepsilon_t, \overline{\pi_t^o}, \omega_t)^{1/r}$.
\end{proof}

\section{Stopping rule}\label{sec:appendix_stopping}
\subsection{Deviation inequality for KL divergences of rewards}
We suppose that the reward distributions $q_{\mcal}(s,a)$ come from a one-dimensional exponential family and can therefore be parametrized by their respective means $r_{\mcal}(s,a)$. Furthermore, for any $t$ such that $N_{sa}(t) > 0$, we let $\widehat{q}_{s,a}(t)$ denote the distribution belonging to the same exponential family, whose mean is the empirical average $\widehat{r}_{t}(s,a) = \frac{\sum\limits_{k=1}^t R_k \indicator\{(s_t,a_t) = (s,a)\}}{N_{sa}(t)}$. For $x\geq 1$, define the function $h(x) = x - \log(x)$ and its inverse $h^{-1}(x)$. Further define the function $\Tilde{h}: \mathbbm{R}^{+} \longrightarrow \mathbbm{R}$ by:
\begin{equation*}
    \Tilde{h}(x) = \begin{cases}
    h^{-1}(x) \exp(1/h^{-1}(x)) \quad \textrm{if } x \geq h^{-1}(1/\ln(3/2)),\\ \\
    \frac{3}{2}\big[x - \log(\log(3/2)\big] \quad \textrm{otherwise.}
    \end{cases}
\end{equation*}
Finally let 
$$\varphi(x) = 2 \Tilde{h}\bigg(\frac{h^{-1}(1+x) + \log(2 \Gamma(2))}{2}\bigg),$$ 
where $\Gamma(2) = \sum_{n=1}^{\infty} 1/n^2$. Now we recall a deviation inequality from \cite{Kaufmann2018MixtureMR}, which we use for the empirical KL divergence of rewards.
\begin{lemma}(Theorem 14, \cite{Kaufmann2018MixtureMR})
Define the threshold $\beta_r(t,\delta) \triangleq SA \varphi\big(\log(1/\delta)/SA\big) + 3\sum\limits_{s,a} \log\big[1+\log(N_{s a}(t))\big]$. Then for all $\delta \in (0,1)$:
\begin{equation*}
   \P\bigg(\exists t \geq 1,\  \underset{(s,a) \in \zcal}{\sum}\ N_{sa}(t) \KL{\widehat{q}_{s,a}(t)}{q_{\mcal}(s,a)} > \beta_p(t, \delta) \bigg) \leq \delta.
\end{equation*}
\label{lemma:concentration_rewards}
\end{lemma}
\begin{remark}
One can easily see that $\varphi(x) \underset{x \to \infty}{\sim} x$.
\end{remark}

\subsection{Deviation inequality for KL divergences of transitions}
Our second deviation inequality is adapted from Proposition 1 in \cite{jonsson2020planning}. There the authors derive a deviation inequality for a \textit{single KL divergence} of  a multinomial distribution. In order to get a deviation inequality of a \textit{sum of KL divergences}, we modified their proof by considering the product over state-action pairs of the martingales they used. For the sake of self-containedness, we include the proof below.
$\widehat{p}_{sa}(t)$ is defined as the categorical distribution with a vector of probabilities  $q$ satisfying:
\begin{align*}
\forall s' \in \scal,\   q_{s'} = 
\begin{cases}
    \displaystyle{\frac{\sum_{k=0}^{t-1} \indicator\{(s_k,a_k,s_{k+1}) = (s,a,s')\}}{N_{sa}(t)}} \quad \textrm{if } N_{sa}(t) \neq 0,\\ \\
    1/A \quad \textrm{otherwise.}
    \end{cases}
\end{align*}
\begin{lemma}(Proposition 1, \cite{jonsson2020planning})
Define the threshold $\beta_p(t,\delta) \triangleq \log(1/\delta) + (S-1)\underset{s,a}{\sum}\log\bigg(e\big[1+N_{sa}(t)/(S-1)\big]\bigg)$. Then for all $\delta \in (0,1)$ we have:
\begin{equation*}
   \P\bigg(\exists t \geq 1,\  \underset{(s,a) \in \zcal}{\sum}\ N_{sa}(t) \KL{\widehat{p}_{sa}(t)}{p_{\mcal}(s,a)} > \beta_p(t, \delta) \bigg) \leq \delta,
\end{equation*}
with the convention that $ N_{sa}(t) \KL{\widehat{p}_{sa}(t)}{p_{\mcal}(s,a)} =0$ whenever $N_{sa}(t)=0$.
\label{lemma:concentration_transitions}
\end{lemma}
\begin{proof}
We begin with a few notations. For any vector $\lambda \in \mathbbm{R}^{S-1}\times \{0\}$ and any element of the simplex $p \in \Sigma_{S-1}$, we denote $<\lambda,p> \triangleq \sum_{i=1}^{S-1} \lambda_i p_i$. We define the log-partition function of a discrete distribution $p$ supported over $\{1,\ldots,S\}$ by: 
\begin{align*}
\forall \lambda \in \mathbbm{R}^{S-1}\times \{0\},\ \phi_p(\lambda) \triangleq \log(p_S + \sum_{i=1}^{S-1} p_i e^\lambda_i).   
\end{align*} 
We use the shorthand and let $\phi_{sa}(\lambda) \triangleq \phi_{p_{sa}}(\lambda)$. For $N \in \mathbbm{N}^*$ and $x \in \{0, \ldots, N\}^k$ such that $\sum_{i=1}^k x_i = N$ the binomial coefficient is defined as: $\binom{N}{x} \triangleq \frac{N!}{\prod_{i=1}^k x_i!}$. Finally $H(p) = \sum_{i=1}^S p_i \log(1/p_i)$ is the Shannon entropy of distribution $p$.
\paragraph{Building a convenient mixture martingale for every state-action pair:} Following \cite{jonsson2020planning}, we define for every integer $t$:
\begin{equation}
    M_t^{\lambda}(s,a) = \exp\bigg(N_{sa}(t)\big(<\lambda,\widehat{p}_{sa}(t)> - \phi_{sa}(\lambda)\big)\bigg)\;.
\end{equation}
The sequence $(M_t^{\lambda}(s,a))_t$  is an $(\mathcal{F}_t)_t$-martingale since:
\begin{align*}
&\E[M_{t}^{\lambda}(s,a) | \mathcal{F}_{t-1},\ (s_t,a_t)=(s,a)] \\
& = \E\bigg[\exp\bigg(N_{sa}(t)\big[<\lambda,\widehat{p}_{sa}(t)> - \phi_{sa}(\lambda)\big]\bigg) \bigg|\ \mathcal{F}_{t-1},\ (s_t,a_t) =(s,a) \bigg]\\
&= \E_{X \sim p_{sa}}\bigg[\exp\bigg(\big(N_{sa}(t-1)+1\big)\big(<\lambda,\frac{N_{sa}(t-1)\widehat{p}_{sa}(t-1) + X}{N_{sa}(t-1)+1}> - \phi_{sa}(\lambda)\big)\bigg)\bigg|\ \mathcal{F}_{t-1}\bigg]\\
& = \E_{X \sim p_{sa}}\bigg[M_{t-1}^{\lambda}(s,a) \exp\bigg(<\lambda,X> - \phi_{sa}(\lambda)\bigg) \bigg|\ \mathcal{F}_{t-1}\bigg] = M_{t-1}^{\lambda}(s,a).
\end{align*}
The same holds trivially when $(s_t,a_t) \neq (s,a)$. Now, we define the mixture martingale defined by the family of priors $\lambda_q = \nabla\phi_{sa}^{-1}(q)$ where $q \sim \mathcal{D}\mathrm{ir}(1, \ldots, 1)$ follows a Dirichlet distribution with parameters $(1, \ldots, 1)$:
\begin{align*}
  & M_t(s,a) = \int M_t^{\lambda_q}(s,a) \frac{\Gamma(S)}{\prod_{i=1}^{S} \Gamma(1)}\prod_{i=1}^{S} q_i dq \\
    &= \int e^{N_{sa}(t) \big(KL(\widehat{p}_{sa}(t),p_{sa}) -KL(\widehat{p}_{sa}(t),q)  \big)} (S-1)! \prod_{i=1}^{S} q_i dq \\
    &=  \exp\bigg(N_{sa}(t) \big(KL(\widehat{p}_{sa}(t),p_{sa}) +H(\widehat{p}_{sa}(t)) \big) \bigg) (S-1)! \int \prod_{i=1}^{S} q_i^{1 + N_{sa}(t)\widehat{p}_{sa,i}(t)} dq  \\
    &= \exp\bigg( N_{sa}(t) \big[KL(\widehat{p}_{sa}(t),p_{sa}) +H(\widehat{p}_{sa}(t))\big]\bigg) \frac{(S-1)! \prod_{i=1}^S \Gamma\bigg(1+N_{sa}(t)\widehat{p}_{sa,i}(t)\bigg)}{\Gamma(N_{sa}(t) + S)}\\
  & = \exp\bigg( N_{sa}(t) \big[KL(\widehat{p}_{sa}(t),p_{sa}) +H(\widehat{p}_{sa}(t))\big]\bigg) \frac{(S-1)! \prod_{i=1}^S \big(N_{sa}(t)\widehat{p}_{sa,i}(t)\big)!}{(N_{sa}(t) + S-1)!}\\
  &= \exp\bigg( N_{sa}(t) \big[KL(\widehat{p}_{sa}(t),p_{sa}) +H(\widehat{p}_{sa}(t))\big]\bigg) \frac{\prod_{i=1}^S \big(N_{sa}(t)\widehat{p}_{sa,i}(t)\big)!}{N_{sa}(t)!} \frac{(S-1)! N_{sa}(t)!}{(N_{sa}(t) + S-1)!}\\
  &= \exp\bigg( N_{sa}(t) \big[KL(\widehat{p}_{sa}(t),p_{sa}) +H(\widehat{p}_{sa}(t))\big]\bigg) \frac{1}{\binom{N_{sa}(t) + S-1}{S-1}}\frac{1}{\binom{N_{sa}(t)}{N_{sa}(t)\widehat{p}_{sa}(t)}},
\end{align*}
where in the second inequality we used Lemma \ref{technical_lemma_1} and $\widehat{p}_{sa,i}(t)$ denotes the i-th component of $\widehat{p}_{sa}(t)$. Now using Lemma \ref{technical_lemma_binomial}, we upper bound the binomial coefficients which leads to: 
\begin{align*}
   M_t(s,a) &\geq \exp\bigg(N_{sa}(t) \big[KL(\widehat{p}_{sa}(t),p_{sa}) +H(\widehat{p}_{sa}(t))\big] - N_{sa}(t) H(\widehat{p}_{sa}(t))\\ & \quad -(N_{sa}(t)+S-1) H(S-1/(N_{sa}(t)+S-1))\bigg) \\
   & = \exp\bigg(N_{sa}(t) KL(\widehat{p}_{sa}(t),p_{sa})-(N_{sa}(t)+S-1) H(S-1/(N_{sa}(t)+S-1))\bigg).
\end{align*}

\paragraph{The product martingale:} Taking the product over all state-action pairs we get:
\begin{align}
    M_t & \triangleq \prod_{(s,a)\in \zcal} M_t(s,a) \nonumber \\
    &\geq \exp\bigg(\sum_{s,a} N_{sa}(t) KL(\widehat{p}_{sa}(t),p_{sa}) - \sum_{s,a} (N_{sa}(t)+S-1) H(S-1/(N_{sa}(t)+S-1))\bigg).
\label{eq:M_t_lb_1}
\end{align}
Next, using that $\log(1+x) \leq x$ we get:
\begin{align*}
(N_{sa}(t)+S-1) H(S-1/(N_{sa}(t)+S-1)) & = (S-1)\log(1 + N_{sa}(t)/(S-1))\\
& \quad + N_{sa}(t)\log(1 + (S-1)/N_{sa}(t))\\
& \leq (S-1)\log(1 + N_{sa}(t)/(S-1)) + (S-1)\\
& = (S-1)\log\bigg(e\big[1+N_{sa}(t)/(S-1)\big]\bigg).
\end{align*}
Hence (\ref{eq:M_t_lb_1}) becomes:
\begin{align}
    M_t \geq \exp\bigg(\sum\limits_{s,a} N_{sa}(t) KL(\widehat{p}_{sa}(t),p_{sa}) - (S-1)\underset{s,a}{\sum}\log\bigg(e\big[1+N_{sa}(t)/(S-1)\big]\bigg) \bigg).
\label{eq:M_t_lb_2}
\end{align}
Now, we show that $M_t$ is a martingale. For any fixed pair $(s,a)$ we have:
\begin{align*}
\E[M_{t} | \mathcal{F}_{t-1},\ (s_t,a_t) =(s,a)] &= \E[M_t(s,a) \prod_{(s',a')\neq (s,a)} M_t(s',a') | \mathcal{F}_{t-1},\ (s_t,a_t) =(s,a)]\\
& = \E[M_t(s,a) \prod_{(s',a')\neq (s,a)} M_{t-1}(s',a') | \mathcal{F}_{t-1}]\\
& = \E[M_{t}(s,a) | \mathcal{F}_{t-1}] \times \prod_{(s',a')\neq (s,a)} M_{t-1}(s',a')\\
& = M_{t-1},
\end{align*}
where the third equality is because $M_t(s,a)$ and $\big(M_{t-1}(s',a')\big)_{(s',a') \neq (s,a)}$ are independent conditionally on $\mathcal{F}_{t-1}$.
Finally, using the tower rule we get:
\begin{align*}
\E[M_{t} | \mathcal{F}_{t-1}] = \E\bigg[\E[M_{t} | \mathcal{F}_{t-1},\ (s_t,a_t)] \bigg| \mathcal{F}_{t-1}\bigg] =  \E[ M_{t-1}| \mathcal{F}_{t-1}] = M_{t-1}.
\end{align*}
Hence $M_t$ is a martingale. Thanks to Doob's maximal inequality we have:
\begin{align*}
    \P\bigg(\exists t \geq 0,\ M_t > 1/\delta \bigg) \leq \delta \E[M_0] = \delta.
\end{align*}
In view of (\ref{eq:M_t_lb_2}), we conclude that for $\beta_p(t,\delta) = \log(1/\delta) +(S-1)\underset{s,a}{\sum}\log\bigg(e\big[1+N_{sa}(t)/(S-1)\big]\bigg)$ we have:
\begin{align*}
   \P\bigg(\exists t \geq 1,\  \sum_{s,a} N_{sa}(t) KL(\widehat{p}_{sa}(t),p_{sa}) > \beta_p(t,\delta) \bigg) \leq \delta.
\end{align*}
\end{proof}
\begin{lemma}(Lemma 3 in \cite{jonsson2020planning})
For $q,p$ in $\Sigma_m$ the simplex of dimension $(m-1)$ and $\lambda \in \mathbbm{R}^{m-1}$, we have:
$$<\lambda,q> - \phi_p(\lambda) = KL(q,p) - KL(q,p^\lambda) $$
where $p^\lambda = \nabla \phi_p(\lambda)$.
\label{technical_lemma_1}
\end{lemma}

\begin{lemma}(Theorem 11.1.3, \cite{elements_IT})
Let $N \in \mathbbm{N}^*$, and $x \in \{0, \ldots, N\}^k$ such that $\sum_{i=1}^k x_i = N$ then:
$$ 
\binom{N}{x} = \frac{N!}{\prod_{i=1}^k x_i!} \leq e^{N H(x/N)} 
$$
where $H(x/N)$ is the Shannon entropy of the discrete distribution over $\{1,\ldots,k\}$ with vector of probabilities $(\frac{x_i}{N})_{1\leq i \leq k}$.
\label{technical_lemma_binomial}
\end{lemma}

\subsection{Correctness of the stopping rule}
\begin{proof}
Using Lemma \ref{lemma:upper_bound} and equations (\ref{eq:GLR_first_exp}-\ref{eq:GLR_second_exp}) in the second and third lines respectively, we get: 
\begin{align*}
\P(\widehat{\pi}_{\tau}^* \neq & \pi^* ,\tau_{\delta} < \infty) = \P\bigg( \exists t \geq 1,\ t\; U\big(\widehat{\mcal}_t,\bN(t)/t\big)^{-1} \geq \beta_r(t,\delta/2) + \beta_p(t,\delta/2), \widehat{\pi}_t^* \neq \pi^* \bigg)\\
& \leq \P\bigg( \exists t \geq 1,\ t\; T\big(\widehat{\mcal}_t,\bN(t)/t\big)^{-1} \geq \beta_r(t,\delta/2) + \beta_p(t,\delta/2), \mcal \in \textrm{Alt}(\widehat{\mcal}_t) \bigg)\\
& = \P\bigg( \exists t \geq 1,\ \underset{\mcal' \in \Alt{\widehat{\mcal}_t}}{\inf} \underset{(s,a) \in \zcal}{\sum}\ N_{s a}(t) \big[\KL{\widehat{q}_{s,a}(t)}{q_{\mcal'}(s,a)} + \KL{\widehat{p}_{sa}(t)}{p_{\mcal'}(s,a)}\big]\\
&\quad \quad \quad \quad \quad \quad \geq \beta_r(t,\delta/2) + \beta_p(t,\delta/2), \mcal \in \textrm{Alt}(\widehat{\mcal}_t) \bigg)\\
& \leq \P\bigg( \exists t \geq 1,\ \underset{(s,a) \in \zcal}{\sum}\ N_{s a}(t) \big[\KL{\widehat{q}_{s,a}(t)}{q_{\mcal}(s,a)} + \KL{\widehat{p}_{sa}(t)}{p_{\mcal}(s,a)}\big]\\
&\quad \quad \quad \quad \quad \quad \geq \beta_r(t,\delta/2) + \beta_p(t,\delta/2) \bigg) \\
& = \P\bigg( \exists t \geq 1,\ \underset{(s,a)\in \zcal}{\sum}\ N_{sa}(t) \KL{\widehat{q}_{s,a}(t)}{q_{\mcal}(s,a)}\geq \beta_r(t,\delta/2)\bigg)\\
& \quad \quad \quad +\P\bigg( \exists t \geq 1,\ \underset{(s,a)\in \zcal}{\sum}\ N_{sa}(t) \KL{\widehat{p}_{sa}(t)}{p_{\mcal}(s,a)} \geq \beta_p(t,\delta/2) \bigg) \\
&\leq \delta/2 + \delta/2 = \delta
\end{align*}
where the last inequality is due to Lemmas \ref{lemma:concentration_rewards} and \ref{lemma:concentration_transitions}.
\end{proof}

\section{Sample complexity upper bound}\label{sec:appendix_SC}
\subsection{Almost sure upper bound: proof of Theorem \ref{theorem:expectation} {\it (i)}}
\begin{proof}
Consider the event $\ecal =\bigg(\forall (s,a) \in \zcal,\ \lim_{t \to \infty}\frac{N_{sa}(t)}{t} = \omega^\star_{s,a},\ \widehat{\mcal}_t \to \mcal \bigg)$. By Lemma \ref{lemma:visits} and Theorem \ref{theorem:ergodic}, we have $\P(\ecal) = 1$. We will prove that under $\ecal$, $\limsup\limits_{\delta \to 0} \frac{\tau_\delta}{\log(1/\delta)} \leq 2 U_{o}(\mcal)$.\\
Fix $\eta > 0$. There exits $t_\eta$ such that for all $t\geq t_\eta$:
\begin{align}
  &U\big(\widehat{\mcal}_t,\bN(t)/t\big)^{-1} \geq (1-\eta) U\big(\mcal,\omega^\star\big)^{-1}\\
  &\beta_p(t,\delta/2) \leq \log(1/\delta) + \eta U\big(\mcal,\omega^\star\big)^{-1} t \\
  &\beta_r(t,\delta/2) \leq  SA\; \varphi\big(\log(1/\delta)/SA\big) + \eta U\big(\mcal,\omega^\star\big)^{-1} t
\end{align}
where the last two inequalities come from the fact that both the thresholds satisfy $\beta(t, \delta/2) = \ocal\big(\log(t)\big) = o(t)$. Combining the inequalities above with the definition of $\tau_\delta$, we get:
\begin{align*}
    \tau_\delta &\leq \inf\bigg\{t \geq t_\eta, (1-3\eta)t U\big(\mcal,\omega^\star\big)^{-1} \geq  \log(1/\delta) +SA\; \varphi\big(\log(1/\delta)/SA\big) \bigg\}\\
    &= \max\bigg(t_\eta,\ \frac{\bigg[log(1/\delta) +SA\; \varphi\big(\log(1/\delta)/SA\big)\bigg]U\big(\mcal,\omega^\star\big) }{1-3\eta}\bigg). 
\end{align*}
Since $\varphi(x) \underset{\infty}{\sim} x$, then the last inequality implies that $\limsup\limits_{\delta \to 0} \frac{\tau_\delta}{\log(1/\delta)} \leq \frac{2 U(\mcal,\omega^\star)}{1-3\eta}$. Taking the limit when $\eta$ goes to zero finishes the proof.
\end{proof}

\subsection{Upper bound in expectation: proof of Theorem \ref{theorem:expectation} {\it (ii)}}
\begin{proof}
We start by defining the semi-distance between MDPs:
\begin{align*}
    \norm{\mcal -\mcal'} = \max\limits_{s,a} \max\big(|r_\mcal(s,a) - r_{\mcal'}(s,a)|, \norm{p_\mcal(.|s,a) - p_{\mcal'}(.|s,a)}_1\big).
\end{align*}
Now for $\xi > 0$, by continuity of $\mcal \to \pi^o(\mcal)$\footnote{As a consequence of Berge's Theorem, which gives the continuity of $\mcal \mapsto \bomega^\star(\mcal)$.} there exists $\rho(\xi) \leq \xi$ such that:
\begin{align*}
    \forall \mcal' \in \bcal\big(\mcal, \rho(\xi)\big),\ \norm{\pi^o(\mcal') - \pi^o(\mcal)}_\infty \leq \xi  
\end{align*}
where $\bcal(\mcal, \rho) = \{\mcal':\ \norm{\mcal' - \mcal} \leq \rho \}$.
For $T \geq 1$, consider the concentration events\footnote{\ For simplicity and w.l.o.g, we consider that $T^{1/4}$ and $T^{3/4}$ are integers.}:
\begin{align*}
    \ccal_T^1(\xi) & \triangleq \bigcap\limits_{t= T^{1/4}}^{T}\bigg(\widehat{\mcal}_t \in \bcal\big(\mcal, \rho(\xi)\big) \bigg).\\
    \ccal_T^2(\xi) &\triangleq \bigcap\limits_{t= T^{3/4}}^{T}\bigg(\big|\bN(t)/t - \omega^\star \big| \leq K_\xi \xi \bigg)
\end{align*}
where $K_\xi$ is a mapping defined in Proposition \ref{proposition:concentration_visits_informal_1}. We will upper bound the stopping time of MDP-NaS under  $\ccal_T^1(\xi) \cap \ccal_T^2(\xi)$. Define:
\begin{align*}
    U(\mcal, \omega^\star, \xi) \triangleq& \sup\  U\big(\mcal', \omega'\big).\\
    &\scalebox{0.75}{$\mcal' \in \bcal(\mcal, \rho(\xi))$}\nonumber\\
    &\scalebox{0.75}{$\norm{\omega' -\omega^\star}_\infty \leq K_\xi \xi$}.  
\end{align*}
By Proposition \ref{proposition:concentration_visits_informal_1}, there exists $T_1(\xi)$ such that for all $T\geq T_1(\xi)$, conditionally on $\ccal_T^1(\xi)$, the event $\ccal_T^2(\xi)$ occurs with high probability. For $T \geq T_1(\xi)$ we have: \begin{align}
    \forall t \in [|T^{3/4}, T|],\ U\big(\widehat{\mcal}_t,\bN(t)/t \big) \leq U(\mcal, \omega^\star, \xi).
\label{ineq:bound_U_voisinage}
\end{align}
Furthermore, using the bound $N_{sa}(t) \leq t$ in the definitions of the thresholds $\beta_p$ and $\beta_r$ and the fact that $\log(t) \underset{t \to \infty}{=} o(t)$, we prove the existence of $T_2(\xi)$ such that for all $t \geq T_2(\xi)$:
\begin{align}
  \beta_p(t,\delta/2) &\leq \log(1/\delta) + \xi U(\mcal,\omega^\star, \xi)^{-1} t \\
  \beta_r(t,\delta/2) &\leq SA\; \varphi\big(\log(1/\delta)/SA\big) + \xi U(\mcal,\omega^\star, \xi)^{-1} t.
\label{ineq:bound_thresholds}
\end{align}
Finally define:
\begin{align*}
  T_3(\xi, \delta) \triangleq \frac{U\big(\mcal,\omega^\star, \xi\big)\bigg[ \log(1/\delta) +SA\varphi\big(\frac{\log(1/\delta)}{SA}\big)\bigg]}{(1-2\xi)}.  
\end{align*}
Using (\ref{ineq:bound_U_voisinage}-\ref{ineq:bound_thresholds}), we have for all $T \geq \max\big(T_1(\xi),T_2(\xi), T_3(\xi, \delta)\big)$ under $\ccal_T^1(\xi) \cap \ccal_T^2(\xi)$ the following holds: 
\begin{align*}
    T\times U\big(\widehat{\mcal}_T,\bN(T)/T\big)^{-1} \geq \beta_p(T,\delta/2)+ \beta_r(T,\delta/2).
\end{align*}
In other words:
\begin{align}
   \forall T \geq \max\big(T_1(\xi),T_2(\xi), T_3(\xi, \delta)\big),\quad \ccal_T^1(\xi) \cap \ccal_T^2(\xi) \subset \big(\tau_\delta \leq T\big).
\label{eq:Golden_inclusion}
\end{align}
Therefore\footnote{\ $\overline{\ecal}$ denotes the complementary of event $\ecal$.}:
\begin{align*}
    \E[\tau_\delta] &= \sum\limits_{T=1}^\infty \P(\tau_\delta > T)\\
    &\leq \max\big(T_1(\xi),T_2(\xi), T_3(\xi, \delta)\big) + \sum\limits_{T=\max(T_1,T_2, T_3)}^\infty \P(\tau_\delta > T)\\
    &\leq \max\big(T_1(\xi),T_2(\xi), T_3(\xi, \delta)\big) + \sum\limits_{T=\max(T_1,T_2, T_3)}^\infty \P\bigg(\overline{\ccal_T^1(\xi)} \cup \overline{\ccal_T^2(\xi)} \bigg)\\ 
    &\leq \max\big(T_1(\xi),T_2(\xi), T_3(\xi, \delta)\big) + \sum\limits_{T=1}^\infty \bigg[\P\bigg(\overline{\ccal_T^1(\xi)}\bigg) +  \P\bigg(\overline{\ccal_T^2(\xi)} \bigg| \ccal_T^1(\xi)\bigg) \bigg] \\
    &\leq \max\big(T_1(\xi),T_2(\xi), T_3(\xi, \delta)\big) + \sum\limits_{T=1}^\infty \frac{1}{T^2} + B T \exp\bigg(-\frac{C T^{1/16}}{\sqrt{\log(1+ SA T^2)}}\bigg) \\
    & \qquad\qquad\qquad\qquad + \frac{2SA \exp(-T^{3/4}\xi^2)}{1-\exp(-\xi^2)},
\end{align*}
where we used Lemma \ref{lemma:concentration_empirical_mdps} and Proposition \ref{proposition:concentration_visits_informal_1} in the last inequality. This implies that $\E[\tau_\delta]$ is finite and:
\begin{align*}
    \limsup\limits_{\delta \to 0} \frac{\E[\tau_\delta]}{\log(1/\delta)} &\leq \limsup\limits_{\delta \to 0} \frac{T_3(\xi, \delta)}{\log(1/\delta)}\\
    &= \frac{2 U\big(\mcal,\omega^\star, \xi\big)}{(1-2\xi)}
\end{align*}
where we used $\varphi(x) \underset{\infty}{\sim} x + \log(x)$. Finally, we take the limit $\xi \to 0$. Since $\rho(\xi) \leq \xi$ and $\limsup\limits_{\xi \to 0} K_\xi < \infty$ then $\limsup\limits_{\xi \to 0} U\big(\mcal,\omega^\star, \xi\big) = U\big(\mcal,\omega^\star\big) = U_{o}(\mcal)$ which finishes the proof. 
\end{proof}

\subsection{Concentration of the empirical MDPs}
\begin{lemma}
Define the event $\ccal_T^1(\xi) \triangleq \bigcap\limits_{t= T^{1/4}}^{T}\bigg(\widehat{\mcal}_t \in \bcal\big(\mcal, \rho(\xi)\big) \bigg)$. Then there exists two positive constants $B$ and $C$ that only depend on $\xi$ and $\mcal$ such that:
\begin{align*}
 \forall T\geq 1,\ \P\left(\overline{\ccal_T^1(\xi)} \right) \leq \frac{1}{T^2} + B T \exp\bigg(-\frac{C T^{1/16}}{\sqrt{\log(1+ SA T^2)}}\bigg).  
\end{align*}
\label{lemma:concentration_empirical_mdps}
\end{lemma}
\begin{proof}
For simplicity we will denote $\rho(\xi)$ by $\rho$. Consider the forced exploration event:
\begin{align*}
    \ecal_T = \bigg(\forall (s,a)\in \zcal,\ \forall t\geq 1,\ N_{sa}(t) \geq \bigg[\frac{t}{\lambda(T)}\bigg]^{1/4} - 1 \bigg)
\end{align*}
where $\lambda(T) \triangleq \frac{(m+1)^2}{\eta^2}\log^2(1+ SA T^2)$ and $\eta$ is a parameter that only depends on $\mcal$. Applying Corollary \ref{corollary:forced_exploration} for $\alpha = \frac{1}{T^2}$, we get $\P(\ecal_T) \geq 1 - 1/T^2$. Therefore we have:
\begin{align}
\P\left(\overline{\ccal_T^1(\xi)} \right) &\leq \P(\overline{\ecal_T}) + \P(\overline{\ccal_T^1(\xi)} \cap \ecal_T)\nonumber \\
&\leq \frac{1}{T^2} + \P(\overline{\ccal_T^1(\xi)} \cap \ecal_T).
\label{ineq:C_1_first_decomposition}
\end{align}
On the other hand:
\begin{align}
 \P(\overline{\ccal_T^1(\xi)} \cap \ecal_T) &\leq \sum_{t = T^{1/4}}^{T} \P\left(\widehat{\mcal}_t \notin \bcal\big(\mcal, \rho(\xi)\big) \cap \ecal_T \right) \nonumber\\
&\leq \sum_{t = T^{1/4}}^{T} \underset{s,a}{\sum}\ \Bigg[\P\bigg(\big(\widehat{r}_t(s,a) - r(s,a) > \rho\big) \cap \ecal_T\bigg) \nonumber \\
& \qquad\qquad + \P\bigg(\big(\widehat{r}_t(s,a) - r(s,a) < -\rho \big) \cap \ecal_T \bigg)\nonumber\\
&\qquad \qquad + \underset{s'}{\sum}\ \P\bigg(\big(\widehat{p}_t(s'|s,a) - p(s'|s,a) > \rho/S\big) \cap \ecal_T \bigg) \nonumber \\
&\qquad\qquad + \P\bigg(\big(\widehat{p}_t(s'|s,a) - p(s'|s,a) < -\rho/S\big) \cap \ecal_T \bigg) \Bigg].
\label{ineq:C_1_second_decomposition}
\end{align}
Using a union bound and Chernoff-Hoeffding theorem respectively we get for $t\geq T^{1/4}$:
\begin{align}
\P\bigg(\big(\widehat{p}_t(s'|s,a) - & p(s'|s,a) > \rho/S\big) \cap \ecal_T \bigg) \nonumber\\
& \leq \P\bigg(\widehat{p}_t(s'|s,a) - p(s'|s,a) > \rho/S,\ N_{sa}(t) \geq \bigg[\frac{t}{\lambda(T)}\bigg]^{1/4} - 1 \bigg) \nonumber\\
&= \sum_{t'=\big[\frac{t}{\lambda(T)}\big]^{1/4} - 1}^{t} \P\bigg(\widehat{p}_t(s'|s,a) - p(s'|s,a) > \rho/S,\ N_{sa}(t) = t' \bigg) \nonumber\\
&\leq  \sum_{t'=\big[\frac{t}{\lambda(T)}\big]^{1/4} - 1}^{t} \exp\bigg(-t'\cdot \kl\big(p(s'|s,a)+\rho/S,\ p(s'|s,a)\big)\bigg)\nonumber\\
&\leq \frac{ \exp\bigg(-\big(\big[\frac{t}{\lambda(T)}\big]^{1/4} - 1 \big)\kl\big(p(s'|s,a)+\rho/S,\ p(s'|s,a)\big)\bigg)}{1 - \exp\bigg(-\kl\big(p(s'|s,a)+\rho/S,\ p(s'|s,a)\big)\bigg)} \nonumber\\
&\leq \frac{ \exp\bigg(-[\frac{T^{1/16}}{\lambda(T)^{1/4}} - 1 ]\kl\big(p(s'|s,a)+\rho/S,\ p(s'|s,a)\big)\bigg)}{1 - \exp\bigg(-\kl\big(p(s'|s,a)+\rho/S,\ p(s'|s,a)\big)\bigg)}.
\label{ineq:C_1_third_decomposition}
\end{align}
In a similar fashion we prove that:
\begin{align}
\P\bigg(\big(\widehat{p}_t(s'|s,a)  & - p(s'|s,a) < -\rho/S\big) \cap \ecal_T \bigg) \nonumber\\
& \qquad\qquad \leq \frac{ \exp\bigg(-[\frac{T^{1/16}}{\lambda(T)^{1/4}} - 1 ]\kl\big(p(s'|s,a)-\rho/S,\ p(s'|s,a)\big)\bigg)}{1 - \exp\bigg(-\kl\big(p(s'|s,a)-\rho/S,\ p(s'|s,a)\big)\bigg)}, 
\end{align}
\begin{align}
\P\bigg(\big(\widehat{r}_t(s,a) - r(s,a) > \rho\big) \cap \ecal_T \bigg) &\leq \frac{ \exp\bigg(-[\frac{T^{1/16}}{\lambda(T)^{1/4}} - 1 ]\kl\big(r(s,a)+\rho,\ r(s,a)\big)\bigg)}{1 - \exp\bigg(-\kl\big(r(s,a)+\rho,\ r(s,a)\big)\bigg)}, \\
\P\bigg(\big(\widehat{r}_t(s,a) - r(s,a) < -\rho\big) \cap \ecal_T \bigg) &\leq \frac{ \exp\bigg(-[\frac{T^{1/16}}{\lambda(T)^{1/4}} - 1 ]\kl\big(r(s,a)-\rho,\ r(s,a)\big)\bigg)}{1 - \exp\bigg(-\kl\big(r(s,a)-\rho,\ r(s,a)\big)\bigg)}.
\label{ineq:C_1_fourth_decomposition}
\end{align}
Thus, for the following choice of constants
\begin{align*}
C = \sqrt{\frac{\eta}{m+1}}\min\limits_{s,a,s'}\Bigg(&\kl\big(r(s,a)-\rho,\ r(s,a)\big),\ \kl\big(r(s,a)+\rho,\ r(s,a)\big),\\
&\kl\big(p(s'|s,a)-\rho/S,\ p(s'|s,a)\big),\  \kl\big(p(s'|s,a)+\rho/S,\ p(s'|s,a)\big)\Bigg),
\end{align*}
\begin{align*}
B = & \sum\limits_{s,a}\ \Bigg(\frac{ \exp\bigg(\kl\big(r(s,a)+\rho,\ r(s,a)\big)\bigg)}{1 - \exp\bigg(-\kl\big(r(s,a)+\rho,\ r(s,a)\big)\bigg)} + \frac{ \exp\bigg(\kl\big(r(s,a)-\rho,\ r(s,a)\big)\bigg)}{1 - \exp\bigg(-\kl\big(r(s,a)-\rho,\ r(s,a)\big)\bigg)} \\
& +\sum\limits_{s'} \Bigg[ \frac{ \exp\bigg(\kl\big(p(s'|s,a)+\rho/S,\ p(s'|s,a)\big)\bigg)}{1 - \exp\bigg(-\kl\big(p(s'|s,a)+\rho/S,\ p(s'|s,a)\big)\bigg)} \\
& + \frac{ \exp\bigg(\kl\big(p(s'|s,a)-\rho/S,\ p(s'|s,a)\big)\bigg)}{1 - \exp\bigg(-\kl\big(p(s'|s,a)-\rho/S,\ p(s'|s,a)\big)\bigg)} \Bigg] \Bigg),
\end{align*}
and using (\ref{ineq:C_1_second_decomposition}-\ref{ineq:C_1_fourth_decomposition}) we get:
\begin{align*}
    \P(\overline{\ccal_T^1(\xi)} \cap \ecal_T) \leq \sum\limits_{t=T^{1/4}}^T B \exp\bigg(-\frac{C T^{1/16}}{\sqrt{\log(1+ SA T^2)}}\bigg) \leq B T \exp\bigg(-\frac{C T^{1/16}}{\sqrt{\log(1+ SA T^2)}}\bigg)\;.
\end{align*}
Combined with (\ref{ineq:C_1_first_decomposition}), the previous inequality implies that:
\begin{align*}
 \P\left(\overline{\ccal_T^1(\xi)} \right) \leq \frac{1}{T^2} + B T \exp\bigg(-\frac{C T^{1/16}}{\sqrt{\log(1+ SA T^2)}}\bigg)  \;.
\end{align*}
\end{proof}

\subsection{Concentration of state-action visitation frequency}
The following proposition is a somewhat stronger version of Proposition \ref{proposition:ergodic_thm}. The fact that $\widehat{\mcal}_t$ is within a distance of at most $\xi$ from $\mcal$ enables us to have a tighter control of the ergodicity constants $L_t$. This in turn allows us to derive a finite sample bound on the deviations of state-action frequency of visits. Before stating the result we recall some simple facts:
\paragraph{Fact 1:} For any two policies $\pi, \pi'$ we have: $\norm{P_{\pi'} - P_{\pi}}_\infty \leq \norm{\pi' - \pi}_\infty$, where the norm is on policies viewed as vectors of $\mathbb{R}^{SA}$.
\paragraph{Fact 2:} Under $\ccal_T^1(\xi)$, for $k\geq \sqrt{T}$ we have:
\begin{align*}
   \norm{\overline{\pi_k^o} -\pi^o(\mcal)}_\infty &\leq \frac{\sum\limits_{j=1}^{T^{1/4}} \varepsilon_j \norm{\pi_u -\pi^o(\mcal)}_\infty}{k} + \frac{\sum\limits_{j=T^{1/4}+1}^{k}  \norm{\pi^o(\widehat{\mcal}_j) -\pi^o(\mcal)}_\infty}{k}\\ 
   &\leq \frac{T^{1/4}}{k} + \xi.
\end{align*}
\paragraph{Fact 3:} Under $\ccal_T^1(\xi)$, for $k\geq \sqrt{T}$ we have:
\begin{align*}
    \norm{\omega_k -\omega^\star}_1 &\leq \kappa_\mcal \norm{P_k - P_{\pi^o}}_\infty\\
    &\leq \kappa_\mcal \norm{\pi_k -\pi^o(\mcal)}_\infty\\
    &\leq \kappa_\mcal \big[\varepsilon_k + \norm{\overline{\pi_k^o} -\pi^o(\mcal)}_\infty \big]\\
    &\leq \kappa_\mcal \big[T^{\frac{-1}{4(m+1)}} + \frac{T^{1/4}}{k} + \xi \big].
\end{align*}
where we used Lemma \ref{lemma:Schweitzer}, Fact 1, the definitions of $\pi_k$ and $\varepsilon_k$ and Fact 2 respectively.
\paragraph{Fact 4:} Under $\ccal_T^1(\xi)$, for $k\geq \sqrt{T}$ we have:
\begin{align*}
    D(\pi_k, \pi_{k-1}) &= \norm{P_k - P_{k-1}}_\infty\\
    &\leq \norm{P_k - P_{\pi^o}}_\infty + \norm{P_{\pi^o} - P_{k-1}}_\infty\\
    &\leq 2\big[T^{\frac{-1}{4(m+1)}} + \frac{T^{1/4}}{k-1} + \xi \big]
\end{align*}
\begin{proposition}
Under C-Navigation, for all $\xi>0$, there exists a time $T_\xi$ such that for all $T \geq T_\xi$, all $t\geq T^{3/4}$ and all functions $f: \zcal \xrightarrow{} \mathbbm{R}^{+}$, we have:
\begin{align*}
  \P\bigg( \bigg|\frac{\sum_{k=1}^{t} f(s_k,a_k)}{t} - \E_{(s,a)\sim\omega^\star}[f(s,a)] \bigg| \geq K_\xi \norm{f}_\infty \xi \bigg| \ccal_T^1(\xi) \bigg) \leq 2\exp\big(-t \xi^2\big),
\end{align*}
where $\xi \mapsto K_\xi$ is a mapping with values in $(1,\infty)$ such that $\limsup_{\xi \to 0} K_\xi < \infty$. In particular, this implies that:
\begin{align*}
    \P\bigg(\exists (s,a) \in \zcal, \bigg|N_{sa}(t)/t - \omega_{s a}^\star \bigg| \geq K_\xi \xi \bigg| \ccal_T^1(\xi) \bigg) \leq 2SA\exp\big(-t \xi^2\big).
\end{align*}
\label{proposition:concentration_visits_informal_1}
\end{proposition}
\begin{corollary}
We have $\P\bigg(\overline{\ccal_T^2(\xi)} \bigg| \ccal_T^1(\xi)\bigg) \leq \displaystyle{\frac{2SA \exp(-T^{3/4}\xi^2)}{1-\exp(-\xi^2)}}$.
\end{corollary}
\begin{proof}
Consider the difference
\begin{align}
    D &= \frac{\sum\limits_{k=1}^{t} f(z_k)}{t} - \omega^\star(f) \nonumber\\
    &= \frac{\sum\limits_{k=1}^{\sqrt{T}} [f(z_k)-\omega^\star(f)]}{t} +  \frac{\sum\limits_{k=\sqrt{T}+1}^{t} [f(z_k)-\omega^\star(f)]}{t}\nonumber\\
    &= \underbrace{\frac{\sum\limits_{k=1}^{\sqrt{T}} [f(z_k)-\omega^\star(f)]}{t} }\limits_{D_{1,t}} + \underbrace{\frac{\sum\limits_{k=\sqrt{T}+1}^{t} \big[f(z_k) - \omega_{k-1}(f)\big]}{t}}\limits_{D_{2,t}} + \underbrace{\frac{\sum\limits_{k=\sqrt{T}+1}^{t} \big[\omega_{k-1}(f) - \omega^\star(f)\big]}{t}}\limits_{D_{3,t}} \nonumber.\\
\label{eq:diff_decomposition_bis}
\end{align}
We clearly have: 
\begin{align}
\forall T \geq \frac{1}{\xi^4},\ \forall t\geq T^{3/4},\ |D_{1,t}| \leq \frac{\norm{f}_\infty \sqrt{T}}{t} \leq \frac{\norm{f}_\infty}{T^{1/4}} \leq \norm{f}_\infty \xi.
\label{ineq:D1_bis}
\end{align}
Using Fact 3 and integral-series comparison, we upper bound the third term as follows:
\begin{align}
\forall T\geq \big(\frac{2}{\xi}\big)^{4(m+1)},\ \forall t\geq T^{3/4},\ |D_{3,t}| &\leq \kappa_\mcal \norm{f}_\infty \frac{\sum\limits_{k=\sqrt{T}+1}^{t-1} \norm{\omega_{k} - \omega^\star}_1 }{t \nonumber}\\
&\leq  \kappa_\mcal \norm{f}_\infty \frac{\sum\limits_{k=\sqrt{T}+1}^{t-1}  \big[T^{\frac{-1}{4(m+1)}} + \frac{T^{1/4}}{k} + \xi \big]}{t}\nonumber \\
&\leq \kappa_\mcal \norm{f}_\infty \bigg( T^{\frac{-1}{4(m+1)}} + \frac{\sum\limits_{k=\sqrt{T}+1}^{t-1} \frac{T^{1/4}}{k}}{t} + \xi \bigg) \nonumber\\
&\leq \kappa_\mcal \norm{f}_\infty \bigg( T^{\frac{-1}{4(m+1)}} + \xi + \frac{T^{1/4}\log(t)}{t} \bigg) \nonumber\\
&\leq \kappa_\mcal \norm{f}_\infty \bigg( T^{\frac{-1}{4(m+1)}} + \xi + \frac{T^{1/4}}{\sqrt{t}} \bigg) \nonumber\\
&\leq \kappa_\mcal \norm{f}_\infty \bigg( T^{\frac{-1}{4(m+1)}} + \xi + T^{\frac{-1}{8}} \bigg) \nonumber\\
&\leq \kappa_\mcal \norm{f}_\infty \bigg( 2T^{\frac{-1}{4(m+1)}} + \xi\bigg) \nonumber\\
&\leq 2\kappa_\mcal \norm{f}_\infty \xi.
\label{ineq:D3_bis}
\end{align}
Now to bound $D_{2,t}$ we use the function $\widehat{f}_k$ solution to the Poisson equation $\big(\widehat{f}_k - P_k\widehat{f}_k\big)(.) = f(.) - \omega_{k}(f)$. By Lemma \ref{lemma:poisson}, $\widehat{f}_k(.) = \sum\limits_{n\geq0} P_k^n[f - \omega_{k}(f)](.)$ exists and is solution to the Poisson equation. Therefore we can rewrite $D_{2,t}$ as follows:
\begin{align}
    D_{2,t} &= \frac{\sum\limits_{k=\sqrt{T}+1}^{t} \big[\widehat{f}_{k-1}(z_k) - P_{k-1}\widehat{f}_{k-1}(z_k)\big]}{t} \nonumber\\
    &= M_t + C_t + R_t.
\label{eq:D_2_decomposition_bis}
\end{align}
where
\begin{align*}
& M_t \triangleq \frac{\sum\limits_{k=\sqrt{T}+1}^{t} \big[\widehat{f}_{k-1}(z_k) - P_{k-1}\widehat{f}_{k-1}(z_{k-1})\big]}{t}.\\ 
& C_t \triangleq  \frac{\sum\limits_{k=\sqrt{T}+1}^{t} \big[P_{k}\widehat{f}_{k}(z_k) - P_{k-1}\widehat{f}_{k-1}(z_k)\big]}{t}.\\
& R_t \triangleq  \frac{P_{\sqrt{T}}\widehat{f}_{\sqrt{T}}(z_{\sqrt{T}}) - P_{t}\widehat{f}_{t}(z_{t})}{t}.
\end{align*}
\paragraph{Bounding $M_t$: }Note that $S_t \triangleq t M_t$ is a martingale since $\E[\widehat{f}_{k-1}(z_k) | \fcal_{k-1}] = P_{k-1}\widehat{f}_{k-1}(z_{k-1})$. Furthermore, by Lemma \ref{lemma:poisson}:
\begin{align}
    |S_k - S_{k-1}| &= |\widehat{f}_{k-1}(z_k) - P_{k-1}\widehat{f}_{k-1}(z_{k-1})| \nonumber\\
    &\leq 2\norm{\widehat{f}_{k-1}}_\infty \nonumber\\
    &\leq 2 \norm{f}_\infty L_{k-1} .
\label{ineq:bound_martingale_difference}
\end{align}
Recall from Lemma \ref{lemma:Geometric_C_Navigation} that $L_k = \lcal(\varepsilon_k, \overline{\pi_k^o}, \omega_k)$ where:
\begin{align*}
    \lcal(\varepsilon, \pi, \omega) &\triangleq \frac{2}{\theta(\varepsilon,\pi,\omega) \big[1- \theta(\varepsilon,\pi,\omega)^{1/r}\big]} \\
    \theta(\varepsilon,\pi,\omega) &\triangleq 1 - \sigma(\varepsilon,\pi,\omega).\\
    \sigma(\varepsilon,\pi,\omega) &\triangleq \bigg[\varepsilon^r + \bigg((1-\varepsilon) A \min\limits_{s,a} \pi(a|s) \bigg)^r \bigg] \sigma_u \bigg(\min\limits_{z}\frac{\omega_u(z)}{\omega(z)}\bigg).
\end{align*}
Now for $T \geq \big(\frac{2}{\xi}\big)^{4(m+1)}$ and $k \geq \sqrt{T}$ we have:
\begin{align*}
    &|\varepsilon_k| = k^{\frac{-1}{2(m+1)}} \leq T^{\frac{-1}{4(m+1)}} \leq \xi/2.\\
    &\norm{\overline{\pi_k^o} -\pi^o(\mcal)}_\infty \leq \frac{T^{1/4}}{k} + \xi \leq \frac{1}{T^{1/4}}+ \xi \leq 2\xi.\\
    &\norm{\omega_k -\omega^\star}_1\leq \kappa_\mcal \big[T^{\frac{-1}{4(m+1)}} + \frac{T^{1/4}}{k} + \xi \big]
    \leq \kappa_\mcal \big[2T^{\frac{-1}{4(m+1)}} + \xi \big] \leq 2\kappa_\mcal \xi.
\end{align*}
Therefore:
\begin{flalign}
    L_k \leq L_\xi \triangleq &\sup\ \lcal(\varepsilon, \pi, \omega) \nonumber.\\
    &\scalebox{0.85}{$|\varepsilon| \leq \xi/2$}\nonumber\\
    &\scalebox{0.85}{$\norm{\pi -\pi^o(\mcal)}_\infty \leq 2\xi$}\nonumber\\
    &\scalebox{0.85}{$\norm{\omega -\omega^\star}_1 \leq 2\kappa_\mcal \xi$}  
\label{ineq:bound_L_k}
\end{flalign}
Using (\ref{ineq:bound_martingale_difference}), (\ref{ineq:bound_L_k}) and Azuma-Hoeffding inequality we get for all $t\geq T^{3/4}$:
\begin{align}
    \P\big(|M_t| \geq 2\norm{f}_\infty L_\xi \xi\big) &= \P\big(|S_t| \geq 2 t\norm{f}_\infty L_\xi \xi\big) \nonumber\\
    &= \P\big( |S_t - S_{\sqrt{T}}| \geq 2 t \norm{f}_\infty L_\xi \xi\big) \nonumber\\
    &\leq 2\exp\bigg(\frac{-t^2 \xi^2}{(t-\sqrt{T})}\bigg)\nonumber\\
    &\leq 2\exp\big(-t \xi^2\big).
 \label{ineq:M_t_bis}
\end{align}

\paragraph{Bounding $C_t$: } Using Lemma \ref{lemma:poisson} we have for all $T\geq \big(\frac{2}{\xi}\big)^{4(m+1)}$ and all $t\geq T^{3/4}$:
\begin{align}
|C_t| &\leq \norm{f}_\infty \frac{\sum\limits_{k=\sqrt{T}+1}^{t} L_k\bigg[ \norm{\omega_k - \omega_{k-1}}_1 + L_{k-1} D(\pi_k, \pi_{k-1}) \bigg] }{t} \nonumber\\
&\leq \norm{f}_\infty \frac{\sum\limits_{k=\sqrt{T}+1}^{t} L_\xi \bigg[  \norm{\omega_k - \omega^\star}_1 + \norm{\omega^\star - \omega_{k-1}}_1 + 2 L_\xi \big(T^{\frac{-1}{4(m+1)}} + \frac{T^{1/4}}{k-1} + \xi \big) \bigg] }{t} \nonumber \\
&\leq \norm{f}_\infty \frac{\sum\limits_{k=\sqrt{T}+1}^{t} L_\xi \bigg[ \kappa_\mcal \big(2T^{\frac{-1}{4(m+1)}} + 2\xi + \frac{T^{1/4}}{k} + \frac{T^{1/4}}{k-1} \big)  + 2 L_\xi \big(T^{\frac{-1}{4(m+1)}} + \frac{T^{1/4}}{k-1} + \xi \big) \bigg]}{t} \nonumber \\
&\leq 2\norm{f}_\infty (\kappa_\mcal L_\xi + L_\xi^2) \bigg[T^{\frac{-1}{4(m+1)}} + \xi + T^{1/4}\frac{\log(t)}{t} \bigg] \nonumber\\
&\leq 2\norm{f}_\infty (\kappa_\mcal L_\xi + L_\xi^2) \bigg[T^{\frac{-1}{4(m+1)}} + \xi + T^{1/4}\frac{1}{\sqrt{t}} \bigg] \nonumber\\
&\leq 2\norm{f}_\infty (\kappa_\mcal L_\xi + L_\xi^2)\bigg[T^{\frac{-1}{4(m+1)}} + \xi + T^{-1/8} \bigg] \nonumber\\
&\leq 4\norm{f}_\infty (\kappa_\mcal L_\xi + L_\xi^2)\xi,
\label{ineq:C_t_bis}    
\end{align}
where the second line comes from (\ref{ineq:bound_L_k}) and Fact 4 and the third line is due to Fact 3.
\paragraph{Bounding $R_t$: } Finally, by Lemma \ref{lemma:poisson} we have:
\begin{align}
\forall T\geq \big(\frac{2}{\xi}\big)^{4(m+1)},\ \forall t\geq T^{3/4},\    |R_t| &\leq \frac{\norm{\widehat{f}_{\sqrt{T}}}_\infty + \norm{\widehat{f}_t}_\infty}{t} \nonumber\\
        &\leq  \frac{\norm{f}_\infty (L_{\sqrt{T}} + L_t)}{t} \nonumber\\
        &\leq 2 \norm{f}_\infty L_\xi T^{-3/4} \nonumber\\
        &\leq 2 \norm{f}_\infty L_\xi \xi.
\label{ineq:R_t_bis}
\end{align}
Summing up the inequalities (\ref{ineq:D1_bis}-\ref{ineq:R_t_bis}) yields for all $T\geq \big(\frac{2}{\xi}\big)^{4(m+1)}$ and all $t\geq T^{3/4}$:
\begin{align}
 \P\bigg(\bigg|\frac{\sum\limits_{k=1}^{t} f(z_k)}{t} - \omega^\star(f) \bigg| \geq K_\xi \norm{f}_\infty \xi \bigg| \ccal_T^1(\xi) \bigg) \leq 2\exp\big(-t \xi^2\big).
 \label{ineq:local_concentration_f}
\end{align}
where $K_\xi \triangleq 1+2\kappa_\mcal + 4 L_\xi(1+\kappa_\mcal+L_\xi)$. Note that $\limsup\limits_{\xi \to 0} L_\xi = \lcal(0, \pi^o(\mcal), \omega^\star) < \infty$ \footnote{Refer to (\ref{eq:sigma_o}) and (\ref{eq:theta_o}) for a formal justification.} implying that $\limsup\limits_{\xi \to 0} K_\xi < \infty$. We get the final result by applying (\ref{ineq:local_concentration_f}) to indicator functions $\indicator_{s,a}(z)$ and using a union bound.
\end{proof}

\section{Technical Lemmas}\label{sec:appendix_technical}
\subsection{Upper bound on the norm of products of substochastic matrices}
Before we proceed with the lemma, we lay out some definitions. $\eta_1 \triangleq \min\big\{P_{\pi_u}(z,z')\ \big| (z,z')\in \zcal^2, P_{\pi_u}(z,z') > 0  \big\}$ denotes the minimum positive probability of transition in $\mcal$. Similarly define $\eta_2 \triangleq \min\big\{P^{n}_{\pi_u}(z,z')\ \big| (z,z')\in \zcal^2, n \in [|1,m+1|], P^{n}_{\pi_u}(z,z') > 0  \big\}$ the minimal probability of reaching some state-action pair $z'$ from any other state-action $z$ after $n\leq m+1$\footnote{Refer to the preamble of Appendix \ref{sec:appendix_sampling} for more detail.} transitions in the Markov chain induced by the uniform random policy. Finally, $\eta \triangleq \eta_1\eta_2$.
\begin{lemma}
 Fix some state-action $z$ and let $P_t$ be the transition matrix under some policy $\pi_t$ satisfying $\pi_t(a|s) \geq \epsilon_t \pi_u(a|s)$ for all $(s,a) \in \zcal$. Define the substochastic matrix $Q_t$ obtained by removing from $P_t$ the row and the column corresponding to $z$: 
 $$P_t = 
\begin{pmatrix}
  \begin{matrix}
  \\
  \quad Q_t \quad \\
  \\
  \end{matrix}
  & \rvline & [P_t(z',z)]_{z' \neq z} \\
\hline
   [P_t(z, z')]_{z' \neq z}^{T} & \rvline &
  \begin{matrix}
    P_t(z,z)
  \end{matrix}
\end{pmatrix}\;.
$$
Then we have:
\begin{equation*}
  \forall n\geq 1,\  \norm{\prod\limits_{l=n+1}^{n+m+1} Q_l}_\infty \leq 1 - \eta \prod\limits_{l=n+1}^{n+m+1} \epsilon_l.
\end{equation*}

\label{lemma:substochastic_matrices}
\end{lemma}
\begin{proof}
Define $r_k(n_1, n_2) = \sum\limits_{j=1}^{SA-1} \bigg(\prod\limits_{l=n_1 + 1}^{n_2} Q_l\bigg)_{k j}$ the sum of the k-th row in the product of matrices $Q_l$ for $l \in [[n_1 +1, n_2|]$. We will prove that or all $i \in [|1,SA-1|]$: $r_i(n, n+m+1) \leq 1 - \eta \prod\limits_{l=n+1}^{n+m+1} \epsilon_l$. The result follows immediately by noting that $\norm{\prod\limits_{l=n+1}^{n+m+1} Q_l}_\infty = \max\limits_{i \in [|1,SA-1|]} r_i(n,n+m+1)$.\\
Consider $z'$ such that $P_{\pi_u}(z',z) \geq  \eta_1$ (such $z'$ always exists since $\mcal$ is communicating) and let $k^\star$ be the index of the row corresponding to $z'$ in $Q_t$. Then for all $n_1\geq1$:
\begin{align}
    r_{k^\star}(n_1,l=n_1+1) &=  \sum\limits_{j=1}^{SA-1} (Q_{n_1+1})_{k^\star j}\nonumber\\
    &= 1 - P_{n_1+1}(z',z)\nonumber\\
    &\leq 1- \eta_1\epsilon_{n_1+1}.
\label{eq:r_k_star}
\end{align}
Now for $n_1, n_2\geq 1$ we have:
\begin{align}
 r_{k^\star}(n_1,n_1+ n_2) &= \sum\limits_{j_1=1}^{SA-1} \bigg(\prod\limits_{l=n_1+1}^{n_1+n_2} Q_l\bigg)_{k^\star j_1} \nonumber\\
 &= \sum\limits_{j_1=1}^{SA-1} \sum\limits_{j_2=1}^{SA-1} \bigg(\prod\limits_{l=n_1+1}^{n_1+n_2 -1}Q_l\bigg)_{k^\star j_2} (Q_{n_1+n_2})_{j_2 j_1}\nonumber \\
 &= \sum\limits_{j_2=1}^{SA-1} \bigg(\prod\limits_{l=n_1+1}^{n_1+n_2 -1}Q_l\bigg)_{k^\star j_2} \bigg[\sum\limits_{j_1=1}^{SA-1} (Q_{n_1+n_2})_{j_2 j_1} \bigg] \nonumber\\
 &= \sum\limits_{j_2=1}^{SA-1} \bigg(\prod\limits_{l=n_1+1}^{n_1+n_2 -1}Q_l\bigg)_{k^\star j_2} r_{j_2}(n_1+n_2-1,n_1+n_2)\nonumber\\
 &\leq r_{k^\star}(n_1,n_1+n_2-1)\nonumber\\
 &\quad \vdots \nonumber\\
 &\leq r_{k^\star}(n_1,n_1+1)\nonumber\\
 &\leq 1 - \eta_1\epsilon_{n_1+1},
\label{eq:6}
\end{align}
where in the fifth line we use the fact that for all $j_2, a, b$: $r_{j_2}(a, b) \leq 1$ since the matrices $Q_l$ are substochastic. The last line comes from (\ref{eq:r_k_star}). Now for all other indexes $i \in [|1,SA-1|]$ we have:
\begin{align}
\forall n_1 \in [|1,m|],\ r_i(n, n+m+1) &= \sum\limits_{j_1=1}^{SA-1} \bigg(\prod\limits_{l=n+1}^{n+n_1}Q_l\ \times \prod\limits_{l=n+n_1+1}^{n+m+1}Q_l\bigg)_{ij_1}\nonumber \\ 
&= \sum\limits_{j_1=1}^{SA-1} \sum\limits_{j_2=1}^{SA-1} \bigg(\prod\limits_{l=n+1}^{n+n_1}Q_l\bigg)_{i j_2} \bigg(\prod\limits_{l=n+n_1+1}^{n+m+1}Q_l\bigg)_{j_2 j_1}\nonumber \\  
&= \sum\limits_{j_2=1}^{SA-1} \bigg(\prod\limits_{l=n+1}^{n+n_1}Q_l\bigg)_{i j_2}\ \sum\limits_{j_1=1}^{SA-1}  \bigg(\prod\limits_{l=n+n_1+1}^{n+m+1 }Q_l\bigg)_{j_2 j_1}\nonumber \\
&= \sum\limits_{j_2=1}^{SA-1} \bigg(\prod\limits_{l=n+1}^{n+n_1 }Q_l\bigg)_{i j_2} r_{j_2}(n+n_1,n+m+1) \nonumber\\
&\leq (1 - \eta_1\epsilon_{n+n_1+1})\bigg(\prod\limits_{l=n+1}^{n+n_1}Q_l\bigg)_{i k^\star} + \sum\limits_{j_2 \neq k^\star} \bigg(\prod\limits_{l=n+1}^{n+n_1}Q_l\bigg)_{i j_2} \nonumber\\
&\leq (1 - \eta_1\epsilon_{n+n_1+1})\bigg(\prod\limits_{l=n+1}^{n+n_1}Q_l\bigg)_{i k^\star} + 1 - \bigg(\prod\limits_{l=n+1}^{n+n_1}Q_l\bigg)_{i k^\star} \nonumber\\
&= 1 - \eta_1\epsilon_{n+n_1+1}\bigg(\prod\limits_{l=n+1}^{n+n_1}Q_l\bigg)_{i k^\star},
\label{eq:7}
\end{align}
where we used (\ref{eq:6}) and the fact that the matrix $\prod\limits_{l=n+1}^{n+n_1}Q_l$ is substochastic. Now since $\mcal$ is communicating then we can reach state-action $z'$ from any other state-action $z_i \in [|1,SA-1|]$,  after some $n_i\leq m+1$ steps in the Markov chain corresponding to the random uniform policy. In other words, if $i$ is the index corresponding to $z_i$ then there exists $n_i \leq m+1$, such that $(P_{\pi_u}^{n_i})_{i k^\star} \geq \eta_2 >0$. Therefore:
\begin{align}
\bigg(\prod\limits_{l=n+1}^{n+n_i}Q_l\bigg)_{i k^\star} &\geq \bigg(\prod\limits_{l=n+1}^{n+n_i} \epsilon_l P_{\pi_u}\bigg)_{i k^\star} \nonumber\\
&= \bigg(\prod\limits_{l=n+1}^{n+n_i} \epsilon_l \bigg)(P_{\pi_u}^{n_i})_{i k^\star}\nonumber\\
&\geq \eta_2 \prod\limits_{l=n+1}^{n+n_i} \epsilon_l.
\label{eq:8}
\end{align}
Thus, combining (\ref{eq:7}) for $n_1 = n_i$ and (\ref{eq:8}) we get: 
\begin{align*}
\forall i \in [|1,SA-1|],\ r_i(n, n+m+1) &\leq 1 - \eta_1\eta_2 \prod\limits_{l=n+1}^{n+n_i} \epsilon_l \\
&\leq 1 - \eta_1\eta_2 \prod\limits_{l=n+1}^{n+m+1} \epsilon_l \\
&= 1 - \eta \prod\limits_{l=n+1}^{n+m+1} \epsilon_l.
\end{align*}
\end{proof}

\subsection{Geometric ergodicity: a general result}
The following lemma is adapted from the proof of the Convergence theorem (Theorem 4.9, \cite{LevinPeresWilmer2006}).
\begin{lemma}
Let $P$ be a stochastic matrix with stationary distribution vector $\omega$. Suppose that there exist $\sigma > 0$ and an integer $r$  such that $P^r(s,s') \geq \sigma \omega(s')$ for all $(s,s')$. Let $W$ be a rank-one matrix whose rows are equal to $\omega\transpose$. Then:
$$
\forall n\geq 1,\ \norm{P^n - W}_\infty \leq 2 \theta^{\frac{n}{r}-1}
$$
where $\theta = 1 - \sigma$.
\label{lemma:geometric_ergodicity}
\end{lemma}
\begin{proof}
We write: $P^r = (1-\theta) W + \theta Q$ where $Q$ is a stochastic matrix. Note that $W P^k = W$  for all $k\geq 0$ since $\omega\transpose = \omega\transpose P$. Furthermore $M W = W$ for all stochastic matrices since all rows of $W$ are equal. Using these properties, we will show by induction that $P^rk = (1-\theta^k) W + \theta^k Q^k$: \\
For $k=1$ the result is trivial. Now suppose that $P^rk = (1-\theta^k) W + \theta^k Q^k$. Then:
\begin{align*}
    P^r(k+1) &= P^rk P^r\\
    &= [(1-\theta^k) W + \theta^k Q^k] P^r\\
    &= (1-\theta^k) W P^r + (1-\theta)\theta^k Q^k W + \theta^{k+1} Q^{k+1}\\
    &= (1 - \theta^k) W + 1-\theta)\theta^k W + \theta^{k+1} Q^{k+1}\\
    &= (1 - \theta^{k+1}) W + \theta^{k+1} Q^{k+1}.
\end{align*}
Therefore the result holds for all $k\geq 1$. Therefore $P^{rk +j} - W = \theta^k (Q^k P^j - W)$ which implies:
\begin{align*}
\forall n = rk+j \geq 1,\  \norm{P^n - W}_\infty &\leq \theta^k \norm{Q^k P^j - W}_\infty\\
&\leq 2 \theta^k = 2 \theta^{\floor{\frac{n}{r}}} \leq 2 \theta^{\frac{n}{r}-1}.
\end{align*}
\end{proof}

\subsection{Condition number of Markov Chains}
\begin{lemma}(Theorem 2 in \cite{Schweitzer68})
Let $P_1$ (resp. $P_2$) be the transition kernel of a Markov Chain with stationary distribution $\omega_1$ (resp. $\omega_2$). Define $Z_1 \triangleq (I -P_1 + \mathbbm{1}\omega_1\transpose)^{-1}$. Then: 
\begin{align*}
    \omega_2\transpose - \omega_1\transpose = \omega_2\transpose [P_2 - P_1] Z_1 \quad \textrm{and}\quad \norm{\omega_2 - \omega_1}_1 \leq \kappa_1 \norm{P_2 - P_1}_\infty.
\end{align*}
where $\kappa_1 \triangleq \norm{Z_1}_\infty$. Crucially, in our setting this implies that there exists a constant $\kappa_\mcal$ that only depends on $\mcal$ such that for all $\pi$:
\begin{align*}
    \norm{\omega_\pi - \omega^\star}_1 \leq \kappa_\mcal \norm{P_\pi - P_{\pi^o}}_\infty. 
\end{align*}
where $\kappa_\mcal \triangleq \norm{Z_{\pi^o}}_\infty = \norm{(I -P_{\pi^o} + \mathbbm{1}{\omega^\star} \transpose)^{-1}}_\infty$.
\label{lemma:Schweitzer}
\end{lemma}

\subsection{Properties of Poisson equation's solutions}
\begin{lemma}
Let $P_\pi$ be a Markov transition kernel satisfying the assumptions \hyperlink{assumption:B1}{(B1)} and \hyperlink{assumption:B3}{(B3)} and denote by $\omega_\pi$ its stationary distribution. Then for any a bounded function $f: \zcal \to \mathbb{R}^{+}$, the function defined by $\widehat{f}_\pi(.) \triangleq \sum\limits_{n\geq0} P_\pi^n[f - \omega_{\pi}(f)](.)$ is well defined and is solution to the Poisson equation $\big(\widehat{f}_\pi - P_\pi\widehat{f}_\pi\big)(.) = f(.) - \omega_{\pi}(f)$. Furthermore:
\begin{equation*}
  \norm{\widehat{f}_\pi}_{\infty} \leq L_{\pi} \norm{f}_{\infty},
\end{equation*}
and for any pair of kernels $P_\pi, P_{\pi'}$:
\begin{equation*}
   \norm{P_{\pi'}\widehat{f}_{\pi'}-P_\pi\widehat{f}_\pi}_{\infty} \leq L_{\pi'} \norm{f}_\infty \big[ \norm{\omega_{\pi'} -\omega_\pi}_1 + L_\pi D(\pi',\pi) \big].
\end{equation*}
\label{lemma:poisson}
\end{lemma}

\begin{proof}
We will prove that $\widehat{f}_\pi$ is well defined. Checking that it satisfies the Poisson equation is straightforward. Observe that:
\begin{align*}
    \widehat{f}_\pi(.) &\triangleq \sum\limits_{n\geq0} P_\pi^n[f - \omega_{\pi}(f)](.)\\
    &= \sum\limits_{n\geq0} [P_\pi^n - W_\pi] [f - \omega_{\pi}(f)](.),\\
\end{align*}
where the second equality is because $(W_\pi f)(z) = \omega_{\pi}(f)$ for all $z\in \zcal$. From the last expression, we see that the sum defining $\widehat{f}_\pi$ converges and we have the first bound $\norm{\widehat{f}_\pi}_{\infty} \leq L_{\pi} \norm{f}_\infty $. Now for the second bound, we write:
\begin{align}
  (P_{\pi'}\widehat{f}_{\pi'}-P_\pi\widehat{f}_\pi)(.) &= \sum\limits_{n\geq 1} \bigg[P_{\pi'}^n \big[\omega_\pi(f) - \omega_{\pi'}(f)\big] + \big[P_{\pi'}^n - P_{\pi}^n \big] \big[f - \omega_\pi(f)\big] \bigg](.)\nonumber\\
  &= \underbrace{\sum\limits_{n\geq 1} P_{\pi'}^n \big[\omega_\pi(f) - \omega_{\pi'}(f)\big](.)}\limits_{A(.)} + \underbrace{\sum\limits_{n\geq 1}\big[P_{\pi'}^n - P_{\pi}^n \big] \big[f - \omega_\pi(f)\big](.)}\limits_{B(.)}.
\label{eq:A+B}
\end{align}
Using the same trick as before we obtain: 
\begin{align}
\norm{A}_\infty &\leq \norm{f}_\infty  C_{\pi'}(1-\rho_{\pi'})^{-1} \norm{\omega_\pi -\omega_{\pi'}}_1 \nonumber\\
&= \norm{f}_\infty L_{\pi'} \norm{\omega_\pi -\omega_{\pi'}}_1.
\label{ineq:A}
\end{align} 
On the other hand, a simple calculation shows that $(B - P_{\pi'}B) (.) = (P_{\pi'} - P_\pi)\widehat{f}_\pi (.)$, ie $B$ is solution to the modified Poisson equation where the right hand side is $(P_{\pi'} - P_\pi)\widehat{f}_\pi$. Therefore:
\begin{align*}
    B(.) =  \sum\limits_{n\geq0} P_{\pi'}^n \big[(P_{\pi'} - P_\pi)\widehat{f}_\pi \big](.).
\end{align*}
and
\begin{align}
   \norm{B}_\infty &\leq L_{\pi'}\norm{(P_{\pi'} - P_\pi)\widehat{f}_\pi}_\infty \nonumber\\
   &\leq  L_{\pi'} D(\pi',\pi) \norm{\widehat{f}_\pi}_\infty  \nonumber\\
   &\leq L_{\pi} L_{\pi'} D(\pi',\pi) \norm{f}_\infty.
\label{ineq:B}
\end{align}
Summing up equation (\ref{eq:A+B}) and inequalities (\ref{ineq:A}-\ref{ineq:B}) ends the proof.
\end{proof}





\end{document}